\documentclass{article}

\usepackage{amsmath}
\usepackage{amssymb}
\usepackage{amsthm}
\usepackage{multirow}
\usepackage{booktabs}

\usepackage[margin=1in]{geometry}
 
\usepackage{tikz}
\usetikzlibrary{shapes,shapes.callouts,shadows,arrows,backgrounds,%
matrix,patterns,arrows,decorations.pathmorphing,decorations.pathreplacing,%
positioning,fit,calc,decorations.text%
}

\usepackage[T1]{fontenc}
\usepackage{textcomp}

\usepackage{enumitem}

\usepackage{amsthm}
\usepackage{amsfonts}
\usepackage{amsmath}
\usepackage{url}

\usepackage{xspace}
\usepackage{boxedminipage}
\usepackage[mathscr]{euscript}

\newtheorem{lemma}{Lemma} 
\newtheorem{proposition}{Proposition} 
\newtheorem{theorem}{Theorem} 
\newtheorem{corollary}{Corollary}
\newtheorem{definition}{Definition} 

\theoremstyle{remark}
\newtheorem{example}{Example}

\def\hy{\hbox{-}\nobreak\hskip0pt} 
\newcommand{\SB}{\{\,}%
\newcommand{\SM}{\;{:}\;}%
\newcommand{\SE}{\,\}}%

\newcommand{\Card}[1]{|#1|}
\newcommand{\CCard}[1]{\|#1\|}
\let\phi=\varphi
\let\epsilon=\varepsilon

\newcommand{\UP}{\text{\normalfont{UP}}}
\newcommand{\Nat}{\mathbb{N}}
 
\newcommand{\AAA}{\mathsf{A}}

\newcommand{\CCC}{\mathscr{C}}

 \newcommand{\TTT}{\mathcal{T}}

\newcommand{\PP}{P}
\newcommand{\QQ}{Q}

\newcommand{\mtext}[1]{\text{\normalfont #1}}

\newcommand{\scope}{\mathit{scope}}
\newcommand{\dom}{\mathit{dom}}

\newcommand{\Prob}{\text{Pr}}
\newcommand{\True}{\mathrm{true}}
\newcommand{\False}{\mathrm{false}}

\renewcommand{\P}{\text{\normalfont P}}
\newcommand{\NP}{\text{\normalfont NP}}
\newcommand{\coNP}{\text{\normalfont coNP}}
\newcommand{\FPT}{\text{\normalfont FPT}}

\newcommand{\PR}[2]{\text{\normalfont\textsc{\bfseries #1}(#2)}}
\newcommand{\RPR}[2]{\text{\normalfont\textsc{\bfseries Rec\hy #1}(#2)}}

\newcommand{\HORN}{\mtext{\sc Horn}}

\newcommand{\TWOCNF}{\mtext{2CNF}}

\newtheorem{boldclaim}{Claim}

\newtheorem{property}{Property}
\newcommand{\set}[1]{\{ #1 \}}

\newcommand{\NV}{\textsc{NValue}\xspace}
\newcommand{\AMNV}{\textsc{AtMost-NValue}\xspace}
\newcommand{\ALNV}{\textsc{AtLeast-NValue}\xspace}
\newcommand{\EGClong}{\textsc{Extended Global Cardinality}\xspace}
\newcommand{\EGC}{\textsc{EGC}\xspace}
\newcommand{\AD}{\textsc{AllDifferent}\xspace}
\newcommand{\nbholes}{\text{\#holes}}
\newcommand{\fpt}{fixed-parameter tractable\xspace}

\newcommand{\mydef}[1]{\emph{#1}} %\marginpar{#1}}
\newcommand{\mydefalt}[2]{\emph{#1}} %\marginpar{#2}}
\newcommand{\myIff}{if and only if\xspace}

\newcommand{\etal}{\emph{et al.}}

\newcommand{\Bessiere}{Bessi{\`e}re\xspace}
\newcommand{\hac}{hyper arc consistent\xspace}
\newcommand{\HAC}{HAC\xspace}

\newcommand{\lep}{\mathsf{l}}
\newcommand{\rep}{\mathsf{r}}
\newcommand{\ivl}{\mathsf{ivl}}

\newcommand{\RedIncl}{\textbf{Red-$\subseteq$}\xspace}
\newcommand{\RedDom}{\textbf{Red-Dom}\xspace}
\newcommand{\RedUnit}{\textbf{Red-Unit}\xspace}

\newcommand{\comment}[1]{\relax} %{\color{red}{[#1]}\color{black}}

\begin{document}

  \title{Guarantees and Limits of Preprocessing in\\Constraint Satisfaction and Reasoning%
\thanks{Preliminary and
    shorter versions of this paper appeared in the proceedings of IJCAI 2011 \cite{GaspersSzeider11} and AAAI 2011 \cite{Szeider11a}.}}
  
  \author{Serge Gaspers\thanks{UNSW Australia and NICTA, Sydney, Australia} \and
          Stefan Szeider\thanks{Institute of Information Systems, Vienna University
    of Technology, Austria}}
  \date{}
   
  \maketitle
   
  \begin{abstract}
    We present a first theoretical analysis of the power of
    polynomial-time preprocessing for important combinatorial problems
    from various areas in AI. We consider problems from Constraint
    Satisfaction, Global Constraints, Satisfiability, Nonmonotonic and
    Bayesian Reasoning under structural restrictions. All these
    problems involve two tasks: (i)~identifying the structure in the
    input as required by the restriction, and (ii)~using the
    identified structure to solve the reasoning task efficiently.  We
    show that for most of the considered problems, task~(i) admits a
    polynomial-time preprocessing to a problem kernel whose size is
    polynomial in a structural problem parameter of the input, in
    contrast to task~(ii) which does not admit such a reduction to a
    problem kernel of polynomial size, subject to a complexity
    theoretic assumption. As a notable exception we show that the
    consistency problem for the \textsc{AtMost-NValue} constraint
    admits a polynomial kernel consisting of a quadratic number of
    variables and domain values.  Our results provide a firm
    worst-case guarantees and theoretical boundaries for the
    performance of polynomial-time preprocessing algorithms for the
    considered problems.
    
  \end{abstract}

% , most of the considered problems cannot be reduced by
%     polynomial-time preprocessing to a problem kernel whose size is
%     polynomial in a structural problem parameter of the input, such as
%     induced width or backdoor size.

\pagestyle{plain}

\section{Introduction}

Many important computational problems that arise in various areas of AI
are intractable. Nevertheless, AI research has been very successful in
developing and implementing heuristic solvers that work well on
real-world instances.  An important component of virtually every solver
is a powerful polynomial-time preprocessing procedure that reduces the
problem input. For instance, preprocessing techniques for the
propositional satisfiability problem are based on Boolean Constraint
Propagation (see, e.g., \cite{EenBiere05}), CSP solvers make use of
various local consistency algorithms that filter the domains of
variables~(see, e.g., \cite{Bessiere06}); similar preprocessing methods
are used by solvers for Nonmonotonic and Bayesian reasoning problems
(see, e.g., \cite{GebserKaufmannNeumannSchaub08,BoltGaag06},
respectively).  The history of preprocessing, like applying reduction
rules to simplify truth functions, can be traced back to the 1950's
\cite{Quine59}. A natural question in this regard is how to measure the
quality of preprocessing rules proposed for a specific problem.

Until recently, no provable performance guarantees for polynomial-time
preprocessing methods have been obtained, and so preprocessing was only
subject of empirical studies.  A possible reason for the lack of
theoretical results is a certain inadequacy of the P vs NP framework for
such an analysis: if we could reduce in polynomial time an instance of
an NP-hard problem just by one bit, then we can solve the entire problem
in polynomial time by repeating the reduction step a polynomial number
of times, and $\P=\NP$ follows.

With the advent of \emph{parameterized complexity}
\cite{DowneyFellows99}, a new theoretical framework became
available that provides suitable tools to analyze the power of
preprocessing. Parameterized complexity considers a problem in a
two-dimensional setting, where in addition to the input size~$n$, a
\emph{problem parameter~$k$} is taken into consideration.  This
parameter can encode a structural aspect of the problem instance. A
problem is called \emph{fixed-parameter tractable} (FPT) if it can be
solved in time $f(k)p(n)$ where $f$ is a function of the parameter $k$
and $p$ is a polynomial of the input size $n$. Thus, for FPT problems,
the combinatorial explosion can be confined to the parameter and is
independent of the input size.  It is known that a problem is
fixed-parameter tractable if and only if every problem input can be
reduced by polynomial-time preprocessing to an equivalent input whose
size is bounded by a function of the parameter
\cite{DowneyFellowsStege99}. The reduced instance is called the
\emph{problem kernel}, the preprocessing is called 
\emph{kernelization}. The power of polynomial-time preprocessing can now
be benchmarked in terms of the size of the kernel.  Once a small kernel
is obtained, we can apply any method of choice to solve the kernel:
brute-force search, heuristics, approximation,
etc.~\cite{GuoNiedermeier07}.  Because of this flexibility a small
kernel is generally preferable to a less flexible branching-based
fixed-parameter algorithm. Thus, small kernels provide an additional
value that goes beyond bare fixed-parameter tractability.

Kernelization is an important algorithmic technique that has become the
subject of a very active field in state-of-the-art combinatorial
optimization~(see, e.g., the references in
\cite{Fellows06,GuoNiedermeier07,LokshtanovMisraSaurabh12,Rosamond10}).  Kernelization can be
seen as a \emph{preprocessing with performance guarantee} that reduces a
problem instance in polynomial time to an equivalent instance, the
\emph{kernel}, whose size is a function of the parameter
\cite{Fellows06,Fomin10,GuoNiedermeier07,LokshtanovMisraSaurabh12}.

Once a kernel is obtained, the time required to solve the instance is
a function of the parameter only and therefore independent of the input
size. While, in general, the time needed to solve an instance does not necessarily
depend on the size of the instance alone, the kernelization view is that
it preprocesses the easy parts of an instance, leaving a core instance encoding
the hard parts of the problem instance.
Naturally one aims at kernels that are as small as possible,
in order to guarantee good worst-case running times as a function of the parameter,
and the kernel size provides a performance guarantee for the preprocessing.
Some NP-hard combinatorial problems such as $k$-\textsc{Vertex Cover}
admit polynomially sized kernels, for others such as $k$-\textsc{Path} an
exponential kernel is the best one can hope for
\cite{BodlaenderDowneyFellowsHermelin09}.

As an example of a polynomial kernel, consider the $k$-\textsc{Vertex Cover}
problem, which, for a graph $G=(V,E)$ and an integer parameter $k$,
is to decide whether there is a set $S$ of at most $k$ vertices such that each edge
from $E$ has at least one endpoint in $S$.
Buss' kernelization algorithm for $k$-\textsc{Vertex Cover} (see \cite{BussGoldsmith93}) computes the set $U$ of
vertices with degree at least $k+1$ in $G$. If $|U|>k$, then reject the instance, i.e.,
output a trivial No-instance (e.g., the graph $K_2$ consisting of one edge and the parameter $0$),
since every vertex cover of size at most $k$ contains each vertex from $U$.
Otherwise, if $G \setminus U$ has more than $k (k-|U|)$ edges, then reject the instance,
since each vertex from $G \setminus U$ covers at most $k$ edges.
Otherwise, output the instance $(G \setminus (U \cup L), k-|U|)$, where $L$ is the set of
degree-0 vertices in $G\setminus U$. This instance has $O(k^2)$ vertices and edges.
Thus, Buss' kernelization algorithm gives a quadratic kernel for $k$-\textsc{Vertex Cover}.

% 
% In general the size of the kernel is exponential in the parameter, but
% many important $\NP$-hard optimization problems such as Minimum Vertex
% Cover, parameterized by solution size, admit \emph{polynomial kernels},
% see, e.g., \cite{BodlaenderDowneyFellowsHermelin09} for references.

\medskip

In previous research several NP-hard AI problems have been shown to be
fixed-parameter tractable. We list some important examples from various
areas:

\begin{enumerate}
\item \sloppypar Constraint satisfaction problems (CSP) over a fixed universe of
  values, parameterized by the induced
  width~\cite{GottlobScarcelloSideri02}.
\item Consistency and generalized arc consistency for intractable global
  constraints, parameterized by the cardinalities of certain sets of
  values \cite{BessiereEtal08}.
\item Propositional satisfiability (SAT), parameterized by the size of
  backdoors~\cite{NishimuraRagdeSzeider04-informal}.
\item Positive inference in Bayesian networks with variables of bounded
  domain size, parameterized by size of loop
  cutsets~\cite{Pearl88,BidyukDechter07}.
\item Nonmonotonic reasoning with normal logic programs, parameterized by
  feedback width~\cite{GottlobScarcelloSideri02}.
\end{enumerate}
% However, only exponential kernels are known for these fundamental AI
% problems.  Can we hope for polynomial kernels? 

\noindent All these problems involve the following two tasks. 
\begin{enumerate}
\item[(i)] \emph{Structure Recognition Task}: identify the structure
  in the input as required by the considered parameter.
\item[(ii)] \emph{Reasoning Task}: use the identified structure to
  solve a reasoning task efficiently.
\end{enumerate}

For most of the considered problems we observe the following pattern:
the Structure Recognition Task admits a polynomial kernel, in contrast
to the Reasoning Task, which does not admit a polynomial kernel,
unless the Polynomial Hierarchy collapses to its third level.

A negative exception to this pattern is the recognition problem for CSPs of
small induced width, which most likely does not admit a polynomial kernel.

A positive exception to this pattern is the \AMNV global constraint,
for which we obtain a polynomial kernel.
As in \cite{BessiereEtal08}, the parameter is the number of holes in the domains of the variables, measuring how close the domains are to being intervals.
More specifically, we
present a \emph{linear time} preprocessing algorithm that reduces an
\AMNV constraint~$C$ with $k$ holes to a consistency-equivalent \AMNV
constraint $C'$ of size polynomial in $k$. In fact, $C'$ has at most
$O(k^2)$ variables and $O(k^2)$ domain values.  We also give an
improved branching algorithm checking the consistency of $C'$ in time
$O(1.6181^k+n)$.  The combination of kernelization and branching
yields efficient algorithms for the consistency and propagation of
\mbox{(\textsc{AtMost}-)}\textsc{NValue} constraints.

\subsection*{Outline}

This article is organized as follows.
Parameterized complexity and kernelization are formally introduced in Section \ref{sec:background}.
Section \ref{sec:tools} describes the tools we use to show that certain parameterized problems do not have polynomial kernels.
Sections \ref{sec:CSP}--\ref{sec:nonmonotonic} prove kernel lower bounds for parameterized problems in the areas of constraint networks, satisfiability, global constraints, Bayesian reasoning, and nonmonotonic reasoning.
Each of these sections also gives all necessary definitions, relevant background, and related work for the considered problems.
In addition, Section \ref{sec:global} describes a polynomial kernel for the consistency problem for the \AMNV constraint parameterized by the number of holes in the variable domains, and an \FPT\ algorithm that uses this kernel as a first step.
The correctness and performance guarantees of the kernelization algorithm are only outlined in Section \ref{sec:global} and proved in detail in \ref{app:kernel}.
The conclusion, Section \ref{sec:conclusion}, broadly recapitulates the results and suggests the study of Turing kernels to overcome the shortcomings of (standard) kernels for many fundamental AI and Resoning problems.

\section{Formal Background}
\label{sec:background}

A \emph{parameterized problem} $\PP$ is a subset of $\Sigma^* \times
\Nat$ for some finite alphabet $\Sigma$. For a problem instance $(x,k)
\in \Sigma^* \times \Nat$ we call $x$ the main part and $k$ the
parameter. We assume the parameter is represented in unary.  For the
parameterized problems considered in this paper, the parameter is a
function of the main part, i.e., $k=\pi(x)$ for a function $\pi$. We
then denote the problem as $\PP(\pi)$, e.g.,  $\PR{$U$-CSP}{width}$
denotes the problem \textbf{$U$-CSP} parameterized by the width of the
given tree decomposition.

A parameterized problem $\PP$ is \emph{fixed-parameter tractable} if
there exists an algorithm that solves any input $(x,k) \in \Sigma^*
\times \Nat$ in time $O(f(k) \cdot p(\Card{x}))$ where $f$ is an
arbitrary computable function of $k$ and $p$ is a polynomial in~$|x|$.

A \emph{kernelization} for a parameterized problem $\PP \subseteq
\Sigma^* \times \Nat$ is an algorithm that, given $(x, k) \in \Sigma^*
\times \Nat$, outputs in time polynomial in $\Card{x}+k$ a pair $(x',
k') \in \Sigma^* \times \Nat$ such that 
\begin{enumerate}
\item $(x,k) \in \PP$ if and only if $(x',k') \in \PP$, and
\item $\Card{x'}+k' \leq g(k)$, where $g$ is an arbitrary computable
  function, called the \emph{size} of the kernel.
\end{enumerate}
In particular, for constant $k$ the kernel has constant size $g(k)$.  If
$g$ is a polynomial then we say that $\PP$ admits a \emph{polynomial
  kernel}.

Every fixed-parameter tractable problem admits a kernel. This can be
seen by the following argument due to Downey \etal~\cite{DowneyFellowsStege99}.
Assume we can decide instances $(x,k)$ of problem $\PP$ in time
$f(k)\Card{x}^{O(1)}$. We kernelize an instance $(x,k)$ as follows. If
$\Card{x}\leq f(k)$ then we already have a kernel of size $f(k)$.
Otherwise, if $\Card{x}> f(k)$, then $f(k)\Card{x}^{O(1)} =
\Card{x}^{O(1)}$ is a polynomial; hence we can decide the instance in
polynomial time and replace it with a small decision-equivalent instance
$(x',k')$. Thus we always have a kernel of size at most $f(k)$.
However, $f(k)$ is super-polynomial for NP-hard problems (unless
$\P=\NP$), hence this generic construction does not provide polynomial
kernels.

We understand \emph{preprocessing} for an NP-hard problem as a
\emph{polynomial-time} procedure that transforms an instance of the
problem to a (possible smaller) solution-equivalent instance of the same
problem.  Kernelization is such a preprocessing with a \emph{performance
  guarantee}, i.e., we are guaranteed that the preprocessing yields a
kernel whose size is bounded in terms of the parameter of the given
problem instance. In the literature also different forms of
preprocessing have been considered. An important one is \emph{knowledge
  compilation}, a two-phases approach to reasoning problems where in a
first phase a given knowledge base is (possibly in exponential time)
preprocessed (``compiled''), such that in a second phase various queries
can be answered in polynomial time~\cite{CadoliDoniniLiberatore02}.

\section{Tools for Kernel Lower Bounds}
\label{sec:tools}

In the sequel we will use recently developed tools to obtain kernel
lower bounds. Our kernel lower bounds are subject to the widely believed
complexity theoretic assumption $\NP \not\subseteq \coNP/\text{\normalfont poly}$.
In other words, the tools
allow us to show that a parameterized problem does not admit a
polynomial kernel unless $\NP \subseteq \coNP/\text{\normalfont poly}$.
In particular, $\NP \subseteq \coNP/\text{\normalfont poly}$ would imply the collapse of the
Polynomial Hierarchy to the third level: $\text{PH} = \Sigma_p^3$ \cite{Papadimitriou94}.

A \emph{composition algorithm} for a parameterized problem $\PP
\subseteq \Sigma^* \times \Nat$ is an algorithm that receives as input a
sequence $(x_1, k),\dots,(x_t, k) \in \Sigma^* \times \Nat$, uses time
polynomial in $\sum_{i=1}^t \Card{x_i}+k$, and outputs $(y,k') \in
\Sigma^* \times \Nat$ with (i)~$(y,k')\in \PP$ if and only if $(x_i,k)
\in \PP$ for some $1 \leq i \leq t$, and (ii)~$k'$ is polynomial in~$k$.
A parameterized problem is \emph{compositional} if it has a composition
algorithm.  With each parameterized problem $\PP\subseteq \Sigma^*
\times \Nat$ we associate a classical problem 
\[\UP[\PP]=\SB x\#1^k \SM (x,k)\in P \SE\] 
where $1$ denotes an arbitrary symbol from $\Sigma$ and
$\#$ is a new symbol not in $\Sigma$.  We call $\UP[\PP]$ the
\emph{unparameterized version} of~$\PP$.

The following result is the basis for our kernel lower bounds.

\begin{theorem}[\cite{BodlaenderDowneyFellowsHermelin09,FortnowSanthanam08}]\label{the:comp}
  Let $\PP$ be a parameterized problem whose unparameterized version is
  $\NP$-complete. If $\PP$ is compositional, then it does not admit a
  polynomial kernel unless $\NP \subseteq \coNP/\text{\normalfont poly}$.
\end{theorem}

Let $\PP,\QQ\subseteq \Sigma^* \times \Nat$ be parameterized
problems. We say that $\PP$ is \emph{polynomial parameter reducible} to
$\QQ$ if there exists a polynomial time computable function $K: \Sigma^*
\times \Nat \rightarrow \Sigma^* \times \Nat$ and a polynomial $p$, such
that for all $(x,k) \in \Sigma^* \times \Nat$ we have (i)~$(x,k) \in
\PP$ if and only if $K(x,k) =(x',k')  \in \QQ$, and (ii)~$k'\leq
p(k)$. The function $K$ is called a \emph{polynomial parameter
  transformation}.

The following theorem allows us to transform kernel lower bounds from
one problem to another.
\begin{theorem}[\cite{BodlaenderThomasseYeo09}]\label{the:trans}
  Let $\PP$ and $\QQ$ be  parameterized problems such that $\UP[\PP]$
  is $\NP$-complete, $\UP[\QQ]$ is in $\NP$, and there is a polynomial
  parameter transformation from $\PP$ to~$\QQ$. If $\QQ$ has a
  polynomial kernel, then $\PP$ has a polynomial kernel.
\end{theorem}

\section{Constraint Networks}
\label{sec:CSP}

\emph{Constraint networks} have proven successful in modeling everyday
cognitive tasks such as vision, language comprehension, default
reasoning, and abduction, as well as in applications such as scheduling,
design, diagnosis, and temporal and spatial reasoning \cite{Dechter10}.
A \emph{constraint network} is a triple $I = (V, U, \CCC)$ where $V$ is
a finite set of variables, $U$ is a finite universe of values, and
$\CCC=\{C_1,\dots,C_m\}$ is set of constraints. Each constraint $C_i$ is
a pair $(S_i,R_i)$ where $S_i$ is a list of variables of length $r_i$
called the \emph{constraint scope}, and $R_i$ is an $r_i$-ary relation
over $U$, called the \emph{constraint relation}. The tuples of $R_i$
indicate the allowed combinations of simultaneous values for the
variables $S_i$. A \emph{solution} is a mapping $\tau :V \rightarrow U$
such that for each $1\leq i \leq m$ and $S_i=(x_1,\dots,x_{r_i})$, we
have $(\tau(x_1),\dots,\tau(x_{r_i}))\in R_i$.  A constraint network is
\emph{satisfiable} if it has a solution.

With a constraint network $I = (V, U, \CCC)$ we associate its
\emph{constraint graph} $G=(V,E)$ where $E$ contains an edge between two
variables if and only if they occur together in the scope of a
constraint.  A \emph{width $w$ tree decomposition} of a graph $G$ is a
pair $(T,\lambda)$ where $T$ is a tree and $\lambda$ is a labeling of
the nodes of $T$ with sets of vertices of $G$ such that the following
properties are satisfied: (i)~every vertex of $G$ belongs to
$\lambda(p)$ for some node $p$ of~$T$; (ii)~every edge of $G$ is is
contained in $\lambda(p)$ for some node $p$ of~$T$; (iii)~For each
vertex $v$ of $G$ the set of all tree nodes $p$ with $v\in \lambda(p)$
induces a connected subtree of $T$; (iv)~$\Card{\lambda(p)}-1\leq w$
holds for all tree nodes~$p$.  The \emph{treewidth} of $G$ is the
smallest $w$ such that $G$ has a width $w$ tree decomposition.  The
\emph{induced width} of a constraint network is the treewidth of its
constraint graph \cite{DechterPearl89}. 
% We note that the
% problem of finding a tree decomposition of width~$w$ is $\NP$-hard but
% fixed-parameter tractable in~$w$ \cite{Bodlaender96}.

%  However, it is
% unlikely that the problem has a polynomial kernel. 

% This follows from 

Kernelization fits perfectly into the context of Constraint Processing
where preprocessing and data reduction (e.g., in terms of local
consistency algorithms, propagation, and domain filtering) are key
methods \cite{Bessiere06,HoeveKatriel06}.

Let $U$ be a fixed universe containing at least two elements.  We
consider the following parameterized version of the constraint
satisfaction problem (CSP).  
\begin{quote}
  $\PR{$U$-CSP}{width}$

  \emph{Instance:} A constraint network $I=(V,U,\CCC)$ and a width $w$ tree
  decomposition of the constraint graph of $I$.

  \emph{Parameter:} The integer $w$.

  \emph{Question:} Is $I$ satisfiable?
\end{quote}

Associated with this problem is also the task of recognizing instances
of small treewidth. We state this problem in form of the following
decision problem.

\begin{quote}
  $\RPR{$U$-CSP}{width}$

  \emph{Instance:} A constraint network $I=(V,U,\CCC)$ and an integer
  $w\geq 0$.

  \emph{Parameter:} The integer $w$.

  \emph{Question:} Does $I$ admit a tree decomposition of width $\leq
  w$?
\end{quote}

It is well known that $\PR{$U$-CSP}{width}$ is fixed-parameter tractable
over any fixed universe
$U$~\cite{DechterPearl89,GottlobScarcelloSideri02} (for generalizations
see \cite{SamerSzeider10a}).  We contrast this
classical result and show that it is unlikely that $\PR{$U$-CSP}{width}$
admits a polynomial kernel, even in the simplest case where $U=\{0,1\}$.
\begin{theorem}\label{the:csp}
  $\PR{$\{0,1\}$-CSP}{width}$ does not admit a polynomial kernel
  unless $\NP \subseteq \coNP/\text{\normalfont poly}$.
\end{theorem}
\begin{proof}
  We show that $\PR{$\{0,1\}$-CSP}{width}$ is compositional.  Let
  $(I_i,T_i)$, $1\leq i \leq t$, be a given sequence of instances of
  $\PR{$\{0,1\}$-CSP}{width}$ where $I_i=(V_i,\{0,1\},\CCC_i)$ is a
  constraint network and $T_i$ is a width $w$ tree decomposition of the
  constraint graph of $I_i$.  We may assume, w.l.o.g., that $V_i \cap
  V_j =\emptyset$ for $1\leq i < j \leq t$ (otherwise we can simply
  change the names of variables).  We form a new constraint network
  $I=(V,\{0,1\},\CCC)$ as follows.  We put $V=\bigcup_{i=1}^t V_i \cup
  \{a_1,\dots,a_t,b_0,\dots,b_t\}$ where $a_i,b_i$ are new variables. We
  define the set $\CCC$ of constraints in three groups.

  \begin{enumerate}
  \item 
  For each $1\leq i \leq t$ and each constraint
  $C=((x_1,\dots,x_r),R)\in \CCC_i$ we add to $\CCC$ a new constraint
  $C'=((x_1,\dots,x_r,a_i),R'))$ where $R'=\SB (u_1,\dots,u_r,0) \SM
  (u_1,\dots,u_r)\in R \SE \cup \{(1,\dots,1)\}$. 

  \item We add $t$ ternary constraints $C_1^*,\dots,C_t^*$ where
  $C_i^*=((b_{i-1},b_{i},a_i),R^*)$ and $R^*=\{(0,0,1)$, $(0,1,0)$,
  $(1,1,1)\}$.

  \item Finally, we add two unary constraints $C^0=((b_0),(0))$ and
  $C^1=((b_t),(1))$ which force the values of $b_0$ and $b_t$ to $0$ and
  $1$, respectively.  
\end{enumerate}

  Let $G, G_i$ be the constraint graphs of $I$ and $I_i$, respectively.
Fig.~\ref{fig:CG} shows an illustration of~$G$ for $t=4$.
\begin{figure}[tbp]
  \centering
 \begin{tikzpicture}
    \small
    \draw
    (0,0)   node (b0) {$b_0$}
    (1,0)   node (b1) {$b_1$}
    (2,0)   node (b2) {$b_2$}
    (3,0)   node (b3) {$b_3$}
    (4,0)   node (b4) {$b_4$}

    (0.5,.8)   node (a1) {$a_1$}
    (1.5,.8)   node (a2) {$a_2$}
    (2.5,.8)   node (a3) {$a_3$}
    (3.5,.8)   node (a4) {$a_4$}

    (0.5-.3,1.8)   coordinate (x1) {}
    (0.5+.3,1.8)   coordinate (y1) {}
    (0.5,1.8)   node (z1) {$\dots$}
    (0.5,2.1)   node (V1) {$V_1$}

    (1.5-.3,1.8)   coordinate (x2) {}
    (1.5+.3,1.8)   coordinate (y2) {}
    (1.5,1.8)   node (z2) {$\dots$}
    (1.5,2.1)   node (V2) {$V_2$}

    (2.5-.3,1.8)   coordinate (x3) {}
    (2.5+.3,1.8)   coordinate (y3) {}
    (2.5,1.8)   node (z3) {$\dots$}
    (2.5,2.1)   node (V3) {$V_3$}

    (3.5-.3,1.8)   coordinate (x4) {}
    (3.5+.3,1.8)   coordinate (y4) {}
    (3.5,1.8)   node (z4) {$\dots$}
    (3.5,2.1)   node (V4) {$V_4$}

    ;

    \draw (b0)--(b1)--(b2)--(b3)--(b4)
    (b0)--(a1)--(b1)--(a2)--(b2)--(a3)--(b3)--(a4)--(b4)
    (a1)--(x1) (a1)--(y1)
    (a2)--(x2) (a2)--(y2)
    (a3)--(x3) (a3)--(y3)
    (a4)--(x4) (a4)--(y4)

    ;

 \end{tikzpicture}%
 \caption{Constraint graph $G$.} \label{fig:CG}
 \end{figure}
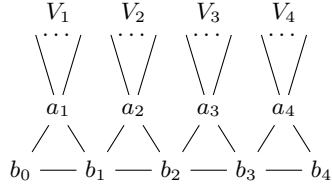
We observe that $a_1,\dots,a_t$ are cut vertices of~$G$. Removing these
vertices separates $G$ into independent parts $P,G_1',\dots,G_t'$ where
$P$ is the path $b_0,b_1,\dots,b_t$, and $G_i'$ is isomorphic to $G_i$.
By standard techniques (see, e.g., \cite{Kloks94}), we can put the given
width $w$ tree decompositions $T_1,\dots,T_t$ of $G_1',\dots,G_t'$ and
the trivial width~1 tree decomposition of $P$ together to a width $w+1$
tree decomposition $T$ of~$G$. Clearly $(I,T)$ can be obtained from
$(I_i,T_i)$, $1\leq i \leq t$, in polynomial time. 

We claim that $I$ is satisfiable if and only if at least one of the
$I_i$ is satisfiable.  
This claim can be verified by means of the
following observations: The constraints in groups (2) and (3) provide
that for any satisfying assignment there will be some $0\leq i \leq t-1$
such that $b_0,\dots,b_i$ are all set to 0 and $b_{i+1},\dots,b_t$ are
all set to $1$; consequently $a_i$ is set to $0$ and all $a_j$ for
$j\neq i$ are set to~1. The constraints in group (1) provide that if we
set $a_i$ to $0$, then we obtain from $C'$ the original constraint $C$;
if we set $a_i$ to $1$ then we obtain a constraint that can be satisfied
by setting all remaining variables to~$1$. We conclude that
$\PR{$\{0,1\}$-CSP}{width}$ is compositional.

 In order to apply Theorem~\ref{the:comp}, we need
  to establish that the unparameterized version of
  $\PR{$\{0,1\}$-CSP}{width}$ is $\NP$-complete. Deciding whether a
  constraint network~$I$ over the universe $\{0,1\}$ is satisfiable is
  well-known to be $\NP$-complete (say by reducing 3-SAT). To a
  constraint network $I$ on $n$ variables we can always add a trivial
  width $w=n-1$ tree decomposition of its constraint graph (taking a
  single tree node $t$ where $\lambda(t)$ contains all variables of
  $I$). Hence $\UP[\PR{$\{0,1\}$-CSP}{width}]$ is $\NP$-complete.
\end{proof}

Let us turn now to the recognition problem $\RPR{$U$-CSP}{width}$. By
Bodlaender's Theorem~\cite{Bodlaender96}, the problem is
fixed-parameter tractable.  However, the problem is unlikely to admit
a polynomial kernel. In fact, Bodlaender et
al.~\cite{BodlaenderDowneyFellowsHermelin09} showed that the related
problem of testing whether a graph has treewidth at most $w$ does not
have a polynomial kernel (taking $w$ as the parameter), unless a
certain ``AND-conjecture'' fails. In turn, Drucker~\cite{Drucker12} showed
that a failure of the AND-conjecture implies $\NP \subseteq
\coNP/\text{\normalfont poly}$. The combination of these two results
relates directly to $\RPR{$U$-CSP}{width}$.

\begin{proposition}
  $\RPR{$\{0,1\}$-CSP}{width}$ does not admit a polynomial kernel
  unless $\NP \subseteq \coNP/\text{\normalfont poly}$.
\end{proposition}

\section{Satisfiability}
The \emph{propositional satisfiability problem} (SAT) was the first
problem shown to be NP-hard \cite{Cook71}. Despite its hardness, SAT
solvers are increasingly leaving their mark as a general-purpose tool in
areas as diverse as software and hardware verification, automatic test
pattern generation, planning, scheduling, and even challenging problems
from algebra \cite{GomesKautzSabharwalSelman08}.  SAT solvers are
capable of exploiting the hidden structure present in real-world problem
instances.  The concept of \emph{backdoors}, introduced by Williams
\etal~\cite{WilliamsGomesSelman03a}, provides a means for making the
vague notion of a hidden structure explicit.  Backdoors are defined with
respect to a ``sub-solver'' which is a polynomial-time algorithm that
correctly decides the satisfiability for a class $\CCC$ of CNF
formulas. More specifically, Gomes
\etal~\cite{GomesKautzSabharwalSelman08} define a \emph{sub-solver} to
be an algorithm $\AAA$ that takes as input a CNF formula $F$ and has the
following properties: 
\begin{enumerate}
\item \emph{Trichotomy}: $\AAA$ either rejects the
input $F$, or determines $F$ correctly as unsatisfiable or satisfiable;
\item \emph{Efficiency}: $\AAA$ runs in polynomial time;
\item \emph{Trivial Solvability}: $\AAA$ can determine if $F$ is
trivially satisfiable (has no clauses) or trivially unsatisfiable
(contains only the empty clause); 
\item  \emph{Self-Reducibility}: if
$\AAA$ determines $F$, then for any variable $x$ and value
$\epsilon\in\{0,1\}$, $\AAA$ determines $F[x=\epsilon]$.  $F[\tau]$
denotes the formula obtained from $F$ by applying the partial assignment
$\tau$, i.e., satisfied clauses are removed and false literals are
removed from the remaining clauses.
\end{enumerate}

We identify a sub-solver $\AAA$ with the class $\CCC_\AAA$ of CNF
formulas whose satisfiability can be determined by $\AAA$.  A
\emph{strong} \emph{$\AAA$-backdoor set} (or \emph{$\AAA$-backdoor}, for
short) of a CNF formula $F$ is a set $B$ of variables such that for each
possible truth assignment $\tau$ to the variables in $B$, the
satisfiability of $F[\tau]$ can be determined by sub-solver $\AAA$ in
time $O(n^c)$. The smaller the backdoor $B$, the more useful it is
for satisfiability solving, which makes the size of the backdoor a natural parameter to consider
(see \cite{GaspersSzeider12} for a survey on the parameterized complexity of backdoor problems).
If we know an $\AAA$-backdoor of size $k$, we can
decide the satisfiability of $F$ by running $\AAA$ on $2^k$ instances
$F[\tau]$, yielding a time bound of $O(2^k n^c)$. Hence SAT decision is
fixed-parameter tractable in the backdoor size~$k$ for any sub-solver
$\AAA$.  Hence the following problem is clearly fixed-parameter
tractable for any sub-solver~$\AAA$.
\begin{quote}
  $\PR{SAT}{$\AAA$-backdoor}$

  \emph{Instance:} A CNF formula $F$, and an $\AAA$-backdoor $B$ of $F$
  of size $k$. 

 \emph{Parameter:} The integer $k$. 

 \emph{Question:} Is
  $F$ satisfiable?
\end{quote}
We also consider for every subsolver $\AAA$ the associated recognition
problem.
\begin{quote}
  $\RPR{SAT}{$\AAA$-backdoor}$

  \emph{Instance:} A CNF formula $F$, and an integer $k \geq 0$. 

 \emph{Parameter:} The integer $k$. 

 \emph{Question:} Does $F$ have an $\AAA$\hy backdoor of size at most
 $k$?
\end{quote}

With the problem $\PR{SAT}{$\AAA$-backdoor}$ we are concerned with the
question of whether instead of trying all $2^k$ possible partial
assignments we can reduce the instance to a polynomial kernel.  We
will establish a very general result that applies to all possible
sub-solvers.
\begin{theorem}\label{the:backdoor}
  $\PR{SAT}{$\AAA$-backdoor}$ does not admit a polynomial kernel for any
  sub-solver $\AAA$ unless $\NP \subseteq \coNP/\text{\normalfont poly}$.
\end{theorem}
\begin{proof}
  We will devise polynomial parameter transformations from the following
  parameterized problem which is known to be compositional
  \cite{FortnowSanthanam08} and therefore  unlikely to admit a
  polynomial kernel.
  \begin{quote}
    $\PR{SAT}{vars}$

    \emph{Instance:} A propositional formula $F$ in CNF on $n$
    variables.

    \emph{Parameter:} The number $n$ of variables.

    \emph{Question:} Is $F$ satisfiable?
  \end{quote}
  Let $F$ be a CNF formula and $V$ the set of all variables of~$F$. Due
  to trivial solvability (Property~3) of a sub-solver, $V$ is an
  $\AAA$-backdoor set for any $\AAA$.  Hence, by mapping $(F,n)$ (as an
  instance of $\PR{SAT}{vars}$) to $(F,V)$ (as an instance of
  $\PR{SAT}{$\AAA$-backdoor}$) provides a (trivial) polynomial parameter
  transformation from $\PR{SAT}{vars}$ to $\PR{SAT}{$\AAA$-backdoor}$.
  Since the unparameterized versions of both problems are clearly
  NP-complete, the result follows by Theorem~\ref{the:trans}.
\end{proof}

Let us denote by $r$CNF the class of CNF formulas where each clause has at most $r$ literals,
and by $\HORN$ the class of CNF formulas where each clause has at most one positive literal.
Sub-solvers for $\HORN$ and $2$CNF follow from \cite{DowlingGallier84} and \cite{Krom67}, respectively.

Let $\PR{3SAT}{$\pi$}$ (where $\pi$ is an arbitrary parameterization)
denote the problem $\PR{SAT}{$\pi$}$ restricted to 3CNF formulas. In
contrast to $\PR{SAT}{vars}$, the parameterized problem
$\PR{3SAT}{vars}$ has a trivial polynomial kernel: if we remove
duplicate clauses, then any 3CNF formula on $n$ variables contains at
most $O(n^3)$ clauses, and so is a polynomial kernel.  Hence the easy
proof of Theorem~\ref{the:backdoor} does not carry over to
$\PR{3SAT}{$\AAA$-backdoor}$. We therefore consider the cases
$\PR{3SAT}{$\HORN$-backdoor}$ and $\PR{3SAT}{$\TWOCNF$-backdoor}$
separately, these cases are important since the detection of $\HORN$ and
$\TWOCNF$-backdoors is fixed-parameter tractable
\cite{NishimuraRagdeSzeider04-informal}.

\begin{theorem}
  Neither $\PR{3SAT}{$\HORN$-backdoor}$ nor
  $\PR{3SAT}{$\TWOCNF$-backdoor}$ admit a polynomial kernel unless $\NP \subseteq \coNP/\text{\normalfont poly}$.
\end{theorem}
\begin{proof}
  Let $\CCC\in \{\HORN,\TWOCNF\}$.  We show that
  $\PR{3SAT}{$\CCC$-backdoor}$ is compositional.  Let $(F_i,B_i)$,
  $1\leq i \leq t$, be a given sequence of instances of
  $\PR{3SAT}{$\CCC$-backdoor}$ where $F_i$ is a 3CNF formula and $B_i$
  is a $\CCC$-backdoor set of $F_i$ of size~$k$.  We distinguish two
  cases. 

  Case 1: $t>2^k$.  Let $\CCard{F_i}:=\sum_{C\in F_i} \Card{C}$ and
  $n:=\max_{i=1}^t \CCard{F_i}$.  Whether $F_i$ is satisfiable or not
  can be decided in time $O(2^kn)$ since the satisfiability of a Horn or
  2CNF formula can be decided in linear time. We can check whether at
  least one of the formulas $F_1,\dots,F_t$ is satisfiable in time
  $O(t2^k n) = O(t^2n)$ which is polynomial in $t+n$.  If some $F_i$
  is satisfiable, we output $(F_i,B_i)$; otherwise we output $(F_1,B_1)$
  ($F_1$ is unsatisfiable). Hence we have a composition algorithm.

  Case 2: $t\leq 2^k$. This case is more involved.  We construct a new
  instance $(F,B)$ of $\PR{3SAT}{$\CCC$-backdoor}$ as follows. 

  Let $s=\lceil \log_2 t \rceil$.  Since $t\leq 2^k$, $s\leq k$ follows.

  Let $V_i$ denote the set of variables of $F_i$. We may assume,
  w.l.o.g., that $B_1=\dots=B_t$ and that $V_i\cap V_j =B_1$ for all
  $1\leq i < j \leq t$ since otherwise we can change names of variable
  accordingly.  In a first step we obtain from every $F_i$ a CNF formula
  $F_i'$ as follows. For each variable $x\in V_i\setminus B_1$ we take
  $s+1$ new variables $x_0,\dots,x_s$. We replace each positive occurrence
  of a variable $x\in V_i\setminus B_1$ in $F_i$ with the literal $x_0$
  and each negative occurrence of $x$ with the literal $\neg x_s$.  
  
  We add
  all clauses of the form $(\neg x_{j-1} \vee x_j)$ for $1\leq j \leq s$;
  we call these clauses ``\emph{connection clauses}.'' Let $F_i'$ be the
  formula obtained from $F_i$ in this way. We observe that $F_i'$ and
  $F_i$ are SAT-equivalent, since the connection clauses form an
  implication chain. Since the connection clauses are both
  Horn and 2CNF, $B_1$ is also a $\CCC$-backdoor of $F_i'$.

  For an illustration of this construction see
  Example~\ref{ex:formulas} below.

  We take a set $Y=\{y_1,\dots,y_s\}$ of new variables. Let
  $C_1,\dots,C_{2^s}$ be the sequence of all $2^s$ possible clauses
  (modulo permutation of literals within a clause) containing each
  variable from $Y$ either positively or negatively.
  Consequently we can write
  $C_i$ as $(\ell^i_1 \vee \dots \vee \ell^i_s)$ where $\ell_j^i\in
  \{y_j,\neg y_j\}$.
  \par
  For $1\leq i \leq t$ we add to each connection clause $(\neg x_{j-1} \vee
  x_j)$
  of $F_i'$ the literal $\ell^i_j\in C_i$. Let $F_i''$ denote the 3CNF
  formula obtained from $F_i'$ this way.
   
  For $t< i \leq 2^s$ we define 3CNF formulas $F_i''$ as follows.  If
  $s\leq 3$ then $F_i''$ consists just of the clause~$C_i$.  If $s>3$
  then we take new variables $z^i_2,\dots,z^i_{s-2}$ and let $F_i''$
  consist of the clauses $(\ell^i_1 \vee \ell^i_2 \vee \neg z^i_2)$,
  $(\ell^i_3 \vee z^i_2 \vee \neg z^i_3), \dots, (\ell^i_{s-2} \vee 
  z^i_{s-3} \vee \neg z^i_{s-2})$, $(\ell^i_{s-1} \vee \ell^i_{s} \vee 
  z^i_{s-2})$.  Finally, we let $F$ be the 3CNF formula containing all
  the clauses from $F_1'',\dots,F_{2^s}''$.  Any assignment $\tau$ to
  $Y\cup B_1$ that satisfies $C_i$  can be extended
  to an assignment that satisfies $F_i''$ since such assignment
  satisfies at least one connection clause $(x_{j-1} \vee x_j \vee
  \ell^i_j)$ and so the chain of implications from from $x_o$ to $x_s$
  is broken. 

  It is not difficult to verify the following two claims. (i)~$F$ is
  satisfiable if and only if at least one of the formulas $F_i$ is
  satisfiable.  (ii)~$B=Y \cup B_1$ is a $\CCC$-backdoor of $F$.  Hence
  we have also a composition algorithm in Case~2, and thus
  $\PR{3SAT}{$\CCC$-backdoor}$ is compositional. Clearly
  $\UP[\PR{3SAT}{$\CCC$-backdoor}]$ is $\NP$\hy complete, hence the
  result follows from~Theorem~\ref{the:comp}.
\end{proof}

\begin{example}\label{ex:formulas}
  We illustrate the constructions of this proof with a running
  example, where we let $s=2$, $t=4$, $i=2$, and $B_1=\{b\}$. 

  \noindent Assume that 
  we have 
  \begin{quote}
  $F_i=
  (x \vee \neg u \vee v) \wedge (\neg x \vee u \vee v) \wedge (\neg x \vee
  \neg u)$.  
  \end{quote}

  \noindent From this we obtain the following formula, containing four
  connection clauses

  \begin{quote}
   $F_i'= (x_0 \vee \neg u_2 \vee
  v) \wedge (\neg x_2 \vee u_0 \vee v) \wedge (\neg x_2 \vee \neg u_2) \wedge$

  \quad $(\neg x_0 \vee x_1) \wedge (\neg x_1 \vee x_2) \wedge (\neg u_0 \vee
  u_1) \wedge (\neg u_1 \vee u_2)$.
  \end{quote}
 Now assume $C_i=(y_1 \vee \neg y_2)$. We add to the
connection clauses literals from $C_i$ and we obtain
\begin{quote}
   $F_i''= (x_0 \vee \neg u_2 \vee
  v) \wedge (\neg x_2 \vee u_0 \vee v) \wedge (\neg x_2 \vee \neg u_2) \wedge$

  \quad $(\neg x_0 \vee x_1 \vee y_1) \wedge (\neg x_1 \vee x_2 \vee
  \neg y_2) \wedge (\neg u_0 \vee
  u_1 \vee y_1) \wedge (\neg u_1 \vee u_2 \vee \neg y_2)$.
\end{quote}

\noindent Assigning $y_1$ to
false and $y_2$ to true reduces $F_i''$  to $F_i'$. The other  three
possibilities of assigning truth values to $y_1,y_2$ break the
connection clauses and make the formula trivially satisfiable. \hfill $\dashv$
\end{example}

We now turn to the recognition problem $\RPR{SAT}{$\AAA$-backdoor}$,
in particular for $\AAA\in \{\HORN,\allowbreak \TWOCNF\}$ for which, as mentioned
above, the problem is known to be fixed-parameter
tractable~\cite{NishimuraRagdeSzeider04-informal}. Here we are able to
obtain positive results.

\begin{proposition}
  Both $\RPR{SAT}{$\HORN$-backdoor}$ and
  $\RPR{SAT}{$\TWOCNF$-backdoor}$ admit polynomial kernels, with 
  a linear and quadratic number of variables, respectively.
\end{proposition}
\begin{proof}
  Let $(F,k)$ be the instance of $\RPR{SAT}{$\HORN$-backdoor}$.  We
  construct a graph $G(F)$ whose vertices are the variables of $F$ and
  which contains an edge between two variables $u,v$ if and only if
  both variables appear as positive literals together in a clause. It
  is well-known and easy to see that the \emph{vertex covers} of
  $G(F)$ are exactly the $\HORN$\hy backdoor sets
  of~$F$~\cite{SamerSzeider08c}. Recall that a vertex cover of a graph
  is a set of vertices that contains at least one end of each edge of
  the graph. Now, we apply the known kernelization algorithm for
  vertex covers \cite{ChenKanjXia10} to $(G(F),k)$ and obtain in
  polynomial time an equivalent instance $(G',k')$ where $G'$ has at
  most $2k$ vertices. Now it only remains to consider $G'$ as a CNF
  formula $F'$ where each edge gives rise to a binary clause on two
  positive literals. Since evidently $G(F')=G'$, we conclude that
  $(F',k')$ constitutes a polynomial kernel for
  $\RPR{SAT}{$\HORN$-backdoor}$.

\sloppypar  For $\RPR{SAT}{$\TWOCNF$-backdoor}$ we proceed similarly.  Let
  $(F,k)$ be an instance of this problem.  We
  construct a 3-uniform hypergraph $H(F)$ whose vertices are the
  variables of $F$ and which contains a hyperedge on any three variables that
  appear (positively or negatively) together in a clause of
  $F$. Again, it is well-known and easy to see that the \emph{hitting
    sets} of $H(F)$ are exactly the $\TWOCNF$\hy backdoor sets
  of~$F$~\cite{SamerSzeider08c}. Recall that a hitting set of a
  hypergraph is a set of vertices that contains at least one vertex from
  each hyperedge. Now we apply a known kernelization algorithm for the
  hitting set problem on 3-uniform hypergraphs (3HS) \cite{Abukhzam10} to
  $(H(F),k)$ and obtain in polynomial time an equivalent instance
  $(H',k')$ where $H'$ has at most $O(k^2)$ vertices. It remains
  to consider $H'$ as a CNF formula $F'$ where each hyperedge gives rise to
  a ternary clause on three positive literals. Since evidently
  $H(F')=H'$, we conclude that $(F',k')$ constitutes a polynomial kernel for
  $\RPR{SAT}{$\TWOCNF$-backdoor}$.
\end{proof}

\section{Global Constraints}
\label{sec:global}
 
Constraint programming (CP) offers a powerful framework for efficient
modeling and solving of a wide range of hard problems
\cite{RossiVanBeekWalsh06}. At the heart of efficient CP solvers are
so-called \emph{global constraints} that specify patterns that
frequently occur in real-world problems. Efficient propagation
algorithms for global constraints help speed up the solver
significantly~\cite{HoeveKatriel06}. For instance, a frequently
occurring pattern is that we require that certain variables must all
take different values (e.g., activities requiring the same resource must
all be assigned different times).  Therefore most constraint solvers
provide a global \AD constraint and algorithms for its
propagation. Unfortunately, for several important global constraints a
complete propagation is NP-hard, and one switches therefore to
incomplete propagation such as bound
consistency~\cite{BessiereEtAl04}. 

In their AAAI'08 paper, Bessi{\`e}re \etal~\cite{BessiereEtal08} showed
that a complete propagation of several intractable constraints can
efficiently be done as long as certain natural problem parameters are
small, i.e., the propagation is \emph{fixed-parameter tractable}
\cite{DowneyFellows99}.  Among others, they showed fixed-parameter
tractability of the \ALNV and \EGClong (\EGC) constraints parameterized
by the number of ``holes'' in the domains of the variables. If there are
no holes, then all domains are intervals and complete propagation is
polynomial by classical results; thus the number of holes provides a way
of \emph{scaling up} the nice properties of constraints with interval
domains.

In the sequel we bring this approach a significant step forward, picking
up a long-term research objective suggested by Bessi{\`e}re
\etal~\cite{BessiereEtal08} in their concluding remarks: whether
intractable global constraints admit a reduction to a problem kernel or
kernelization.

More formally, a global constraint is defined for a set $S$ of
variables, each variable $x\in S$ ranges over a finite domain $\dom(x)$
of values.
For a set $X$ of variables we write $\dom(X)=\bigcup_{x\in X} \dom(x)$.
An \emph{instantiation} is an assignment $\alpha: S \rightarrow \dom(S)$ such that
$\alpha(x)\in \dom(x)$ for each $x\in S$.  A global constraint defines
which instantiations are legal and which are not.
This definition is usually implicit, as opposed to classical constraints, which list all legal tuples.
Examples of global constraints include:

\begin{enumerate}
\item  The global constraint \textsc{NValue} is defined  over
a set $X$ of variables and a variable $N$ and requires from a legal instantiation $\alpha$ that
$\Card{ \SB \alpha(x) \SM x\in X \SE}=\alpha(N)$;

\item  The global constraint \AMNV is defined for fixed values of $N$ over
a set $X$ of variables and requires from a legal instantiation $\alpha$ that
$\Card{ \SB \alpha(x) \SM x\in X \SE}\le N$;

\item The global constraint
\textsc{Disjoint} is specified by two sets of
variables $X,Y$ and  requires that $\alpha(x)\neq
\alpha(y)$ for each pair $x\in X$ and $y\in Y$;

\item The global constraint \textsc{Uses} is also specified by two sets
  of variables $X,Y$ and requires that for each $x\in X$ there is some
  $y \in Y$ such that $\alpha(x)=\alpha(y)$.

\item The global constraint \EGC is specified by a set of variables $X$, a set of values $D=\dom(X)$, and a finite domain $\dom(v)\subseteq \mathbb{N}$ for each value $v\in D$, and it requires that for each $v\in D$ we
  have $\Card{ \SB \alpha(x)=v \SM x\in X \SE}\in \dom(v)$.
\end{enumerate}

A global constraint $C$ is \emph{consistent} if there is a legal
instantiation of its variables. The constraint $C$ is \emph{\hac} (\emph{\HAC{}}) if
for each variable $x\in \scope(C)$ and each value $v\in \dom(x)$, there
is a legal instantiation $\alpha$ such that $\alpha(x)=v$ (in that
case we say that $C$ supports $v$ for $x$). In the literature, \HAC is also called
\emph{domain consistent} or \emph{generalized arc consistent}.
The constraint $C$ is \emph{bound consistent} if when a variable $x\in \scope(C)$
is assigned the minimum or maximum value of its domain, there are compatible values
between the minimum and maximum domain value for all other variables in $\scope(C)$.
The main algorithmic problems
for a global constraint $C$ are the following: %The problem
\emph{Consistency}, to decide whether $C$ is consistent, and
\emph{Enforcing \HAC{}}, to remove from all domains the values that are not
supported by the respective variable.

It is clear that if \HAC can be enforced in polynomial time for a constraint
$C$, then the consistency of $C$ can also be decided in polynomial time
(we just need to see if any domain became empty).  The reverse is true
if for each $x\in \scope(C)$ and $v\in \dom(x)$, the consistency of $C\wedge (x\leftarrow v)$,
requiring $x$ to be assigned the value $v$, can be decided in polynomial time (see
\cite[Theorem 17]{HoeveKatriel06}). This is the case for most constraints of
practical use, and in particular for all constraints considered below. The
same correspondence holds with respect to fixed-parameter
tractability. Hence, we will focus mainly on Consistency.

For several important types $\TTT$ of global constraints, the problem of
deciding whether a constraint of type $\TTT$ is consistent is NP-hard. This includes
the 5 global constraints \textsc{NValue}, \AMNV, \textsc{Disjoint}, \textsc{Uses}, and \textsc{EGC}
defined above (see \cite{BessiereEtAl04}).

Each global constraint of type $\TTT$ and parameter \textsl{par} gives
rise to a parameterized problem:
  \begin{quote}
    $\TTT$-\PR{Cons}{\textsl{par}}

    \emph{Instance:} A global constraint $C$ of type $\TTT$.

    \emph{Parameter:} The integer \textsl{par}.

    \emph{Question:} Is $C$ consistent?
  \end{quote}

\Bessiere \etal~\cite{BessiereEtal08} considered $dx=\Card{\dom(X)}$ as
parameter for \textsc{NValue}, $dxy=\Card{\dom(X) \cap \dom(Y)}$ as
parameter for \textsc{Disjoint}, and $dy=\Card{\dom(Y)}$ as parameter
for \textsc{Uses}.  They showed that consistency checking is
fixed-parameter tractable for the constraints under the respective
parameterizations, i.e., the problems \textsc{NValue}\hy
$\PR{Cons}{$dx$}$, \textsc{Disjoint}\hy $\PR{Cons}{$dxy$}$, and
\textsc{Uses}\hy $\PR{Cons}{$dy$}$ are fixed-parameter tractable.

\newcommand{\holes}{\mtext{\slshape holes}}
\Bessiere \etal~\cite{BessiereEtal08} also showed that polynomial time algorithms for enforcing bounds consistency imply that the corresponding consistency problem is fixed-parameter tractable parameterized by the number of holes. This is the case for the global constraints \textsc{NValue}, \AMNV, and \EGC. 
\begin{definition}
When $D$ is totally ordered, a \mydef{hole} in a subset $D'\subseteq D$ is a couple $(u,w) \in D'
\times D'$, such that there is a $v \in D \setminus D'$ with $u<v<w$ and
there is no $v'\in D'$ with $u<v'<w$.
\end{definition}

We denote the number of holes in
the domain of a variable $x\in X$ by $\nbholes(x)$.  The parameter of
the consistency problem for \AMNV constraints is $\holes= \sum_{x\in X} \nbholes(x)$.

\subsection{Kernel Lower Bounds}
We show that it is unlikely that most of the FPT results of \Bessiere
\etal~\cite{BessiereEtal08} can be improved to polynomial
kernels.

\begin{theorem} The problems \textsc{NValue}\hy $\PR{Cons}{$dx$}$,
  \textsc{Disjoint}\hy $\PR{Cons}{$dxy$}$, \textsc{Uses}\hy
  $\PR{Cons}{$dy$}$ do not admit  polynomial kernels unless $\NP \subseteq \coNP/\text{\normalfont poly}$.
\end{theorem}
\begin{proof} We devise a polynomial parameter transformation from
  $\PR{SAT}{vars}$.  We use a construction of~\Bessiere \etal~\cite{BessiereEtAl04}.
  Let $F= \{C_1,\dots,C_m\}$ be a CNF formula over variables
  $x_1,\dots,x_n$.  We consider the clauses and variables of $F$ as the
  variables of a global constraint with domains $\dom(x_i)=\{-i,i\}$,
  and $\dom(C_j)=\SB i \SM x_i\in C_j \SE \cup \SB -i \SM \neg x_i\in
  C_j \SE$. Now $F$ can be encoded as an \textsc{NValue} constraint with
  $X=\{x_1,\dots,x_n,C_1,\dots,C_m\}$ and $\dom(N)=\set{n}$.
  By the pigeonhole principle, a legal instantiation $\alpha$ for this
  constraint has $|\set{\alpha(x_i) : 1\le i\le n}|=N$. Setting
  $\alpha(x_i)=i$ corresponds to setting the variable $x_i$ of $F$
  to 1 and setting $\alpha(x_i)=-i$ corresponds to setting the variable
  $x_i$ of $F$ to 0. Now, for each $C_j \in F$, $\alpha(C_j) \in
  \set{\alpha(x_i) : 1\le i\le n}$ since only $n$ values are available
  for $\alpha$, and the literal corresponding to $\alpha(C_j)$ satisfies the
  clause $C_j$.
  %(clearly $F$ is satisfiable if and only if the constraint is consistent).
  Since $dx=2n$ we have a polynomial parameter reduction from $\PR{SAT}{vars}$
  to \textsc{NValue}\hy $\PR{Cons}{$dx$}$.  Similarly, as observed by
  \Bessiere \etal~\cite{BessiereHebrardHnichKiziltanWalsh09}, $F$ can be encoded as a
  \textsc{Disjoint} constraint with $X=\{x_1,\dots,x_n\}$ and
  $Y=\{C_1,\dots,C_m\}$ ($dxy\leq 2n$), or as a \textsc{Uses} constraint
  with $X=\{C_1,\dots,C_m\}$ and $Y=\{x_1,\dots,x_n\}$ ($dy=2n$).
  Since the unparameterized problems are clearly NP-complete,
  and $\PR{SAT}{vars}$ is known to have
  no polynomial kernel unless $\NP \subseteq \coNP/\text{\normalfont poly}$
  (as remarked in the proof of Theorem~\ref{the:backdoor}),
  the result
  follows by Theorem~\ref{the:trans}.
\end{proof}
%\noindent Further results on kernels for global constraints
%have been obtained by Gaspers and Szeider \cite{GaspersSzeider11}.

%\section{Extended Global Cardinality Constraints}

The Consistency problem for \EGC constraints is $\NP$-hard
\cite{QuimperLopezortizVanbeekGolynski04}. However, if all sets
$\dom(\cdot)$ are intervals, then consistency can be checked in polynomial
time using network flows~\cite{Regin96}.  By the result of \Bessiere
\etal~\cite{BessiereEtal08}, the Consistency problem for \EGC
constraints is fixed-parameter tractable, parameterized by the number of
holes in the sets $\dom(\cdot)$. Thus R\'{e}gin's result generalizes to
instances that are close to the interval case.

However, it is unlikely that \EGC constraints admit a polynomial kernel.
\begin{theorem}
  \EGC-$\PR{Cons}{\holes}$ does not admit a polynomial kernel unless
   $\NP \subseteq \coNP/\text{\normalfont poly}$.
%  $\NP \subseteq \coNP/\text{\normalfont poly}$.
\end{theorem}
\begin{proof}
  We use the following result of Quimper \etal~\cite{QuimperLopezortizVanbeekGolynski04}: Given a CNF formula $F$ on
  $k$ variables, one can construct in polynomial time an \EGC constraint
  $C_F$ such that
  \begin{enumerate}
  \item for each value $v$ of $C_F$, $\dom(v)=\{0,i_v\}$ for an
    integer $i_v>0$,
  \item $i_v>1$ for at most $2k$ values $v$, and
  \item $F$ is satisfiable \myIff $C_F$ is consistent.
  \end{enumerate}
  Thus, the number of holes in $C_F$ is at most twice the number of
  variables of $F$.

  We observe that this result provides a polynomial parameter reduction
  from $\PR{SAT}{vars}$ to \EGC-$\PR{Cons}{\holes}$.  As remarked in the
  proof of Theorem~\ref{the:backdoor}, $\PR{SAT}{vars}$ is known to have
  no polynomial kernel unless $\NP \subseteq \coNP/\text{\normalfont poly}$.  Hence
  the theorem follows.
  \end{proof}

\subsection{A Polynomial Kernel for NValue Constraints}

Beldiceanu~\cite{Beldiceanu01} and
Bessi{\`e}re \etal~\cite{BessiereEtal06} decompose \NV constraints into
two other global constraints: \AMNV and \ALNV, which require
that %the number of values used for the variables in $X$ is at most $N$ or at least $N$,
at most $N$ or at least $N$ values are used for the variables in $X$,
respectively.
%Checking the consistency of \NV and
%\AMNV constraints is $\NP$-complete, while the Consistency problem of
%\ALNV constraints can be solved in polynomial time.
The Consistency problem is $\NP$-complete for \NV and
\AMNV constraints, and polynomial time solvable for 
\ALNV constraints.

In this subsection, we will present a polynomial kernel for \AMNV-\PR{Cons}{\holes}.

  \begin{quote}
    \AMNV-\PR{Cons}{\holes}

    \emph{Instance:} An instance $\mathcal{I} = (X,D,\dom,N)$, where $X = \set{x_1,\dots,x_n}$ is a set of variables, $D$ is a totally ordered set of values, $\dom:X\rightarrow 2^D$ is a map assigning a non-empty domain $\dom(v)\subseteq D$ to each variable $x\in X$, and an integer $N$.

    \emph{Parameter:} The integer $k = \nbholes(X)$.

    \emph{Question:} Is there a set $S\subseteq D$, $|S| \le N$, such that for every variable $x\in X$, $\dom(x) \cap S \ne \emptyset$?
  \end{quote}

\begin{theorem}\label{thm:kernel}
  The problem \AMNV-$\PR{Cons}{\holes}$ has a polynomial
  kernel. In particular, an \AMNV constraint with $k$ holes can be
  reduced in linear time to a consistency-equivalent \AMNV constraint
  with $O(k^2)$ variables and $O(k^2)$ domain values.
\end{theorem}

The proof of the theorem is based on a kernelization algorithm that we
will describe in the remaining part of this section.

% The \NV constraint was introduced by Pachet and Roy~\cite{PachetRoy99}.
% For a set of variables $X$ and a variable $N$, $\NV(X,N)$ is consistent
% if there is an assignment $\alpha$ such that exactly $\alpha(N)$
% different values are used for the variables in $X$.  \AD is the special
% case where $\dom(N)=\set{|X|}$.  

\smallskip

We say that a subset of $D$ is an \emph{interval} if it has no hole.
An \emph{interval} $I=[v_1,v_2]$ of a variable $x$ is an inclusion-wise maximal hole-free subset of its domain. Its \mydefalt{left endpoint}{left endpoint} $\lep(I)$ and \mydefalt{right endpoint}{right endpoint} $\rep(I)$ are the values $v_1$ and $v_2$, respectively.
Fig.~\ref{fig:input} gives an example of an instance and its interval representation. We assume that instances are given by a succinct description, in which the domain of a variable is given by the left and right endpoint of each of its intervals. As the number of intervals of the instance $\mathcal{I}=(X,D,\dom,N)$ is $n+k$, its size is $|\mathcal{I}|=O(n+|D|+k)$. In case $\dom$ is given by an extensive list of the values in the domain of each variable, a succinct representation can be computed in linear time.

Also, in a variant of \AMNV-\PR{Cons}{\holes} where $D$ is not part of the input, we may construct $D$ by sorting the set of all endpoints of intervals in time $O((n+k) \log(n+k))$. Since, w.l.o.g., a solution
contains only endpoints of intervals, this step does not compromise the correctness.

A greedy algorithm by Beldiceanu~\cite{Beldiceanu01} checks the
consistency of an \AMNV constraint in linear time when all domains are
intervals (i.e., $k=0$).  Further, \Bessiere \etal~\cite{BessiereEtal08}
have shown that Consistency (and Enforcing \HAC) is \fpt, parameterized
by the number of holes, for all constraints for which bound consistency
can be enforced in polynomial time.  A simple algorithm for checking the
consistency of \AMNV goes over all instances obtained from restricting
the domain of each variable to one of its intervals, and executes the
algorithm of \cite{Beldiceanu01} for each of these $2^k$ instances. The
running time of this algorithm is clearly bounded by $O(2^k \cdot
|\mathcal{I}|)$.

\begin{figure} \centering
 \begin{tikzpicture}[xscale=0.65,yscale=0.75]
  \tikzset{ivl/.style={thick},
  var/.style={midway,above=0.5pt}}
  \tikzset{ivlopt/.style={thick,red}}
 
  \foreach \x in {1,...,14} \draw (\x,-0.6) node {$\x$};
  \foreach \x in {1,...,14} \draw[thin,black!20] (\x,-0.3) -- (\x,3.3);
  
  \draw (2.5,0.5) node[rectangle,draw] {$N=6$};
  
  \draw[ivl,|-|] (0.9,3)-- node[var] {$x_1$} (2.1,3);
  \draw[ivlopt,|-] (2.9,3)-- node[var] {$x_3$} (4.1,3);
  \draw[ivl,|-|] (4.9,3)-- node[var] {$x_6$} (7.1,3);
  \draw[ivl,|-|] (7.9,3)-- node[var] {$x_9$} (9.1,3);
  \draw[ivlopt,-|] (9.9,3)-- node[var] {$x_3'$} (11.1,3);
  \draw[ivlopt,|-] (11.9,3)-- node[var] {$x_{13}$} (12.1,3);
  \draw[ivlopt,-|] (13.9,3)-- node[var] {$x_{13}'$} (14.1,3);
  \draw[ivlopt,|-] (1.9,2)-- node[var] {$x_2$} (3.1,2);
  \draw[ivl,|-|] (3.9,2)-- node[var] {$x_4$} (6.1,2);
  \draw[ivl,|-|] (6.9,2)-- node[var] {$x_7$} (8.1,2);
  \draw[ivlopt,-|] (9.9,2)-- node[var] {$x_2'$} (10.1,2);
  \draw[ivl,|-|] (10.9,2)-- node[var] {$x_{12}$} (12.1,2);
  \draw[ivl,|-|] (12.9,2)-- node[var] {$x_{15}$} (14.1,2);
  \draw[ivl,|-|] (5.9,1)-- node[var] {$x_8$} (8.1,1);
  \draw[ivl,|-|] (8.9,1)-- node[var] {$x_{10}$} (10.1,1);
  \draw[ivlopt,|-] (10.9,1)-- node[var] {$x_{11}$} (11.1,1);
  \draw[ivlopt,-|] (12.9,1)-- node[var] {$x_{11}'$} (13.1,1);
  \draw[ivlopt,|-] (4.9,0)-- node[var] {$x_5$} (6.1,0);
  \draw[ivlopt,-|] (7.9,0)-- node[var] {$x_5'$} (10.1,0);
  \draw[ivl,|-|] (11.9,0)-- node[var] {$x_{14}$} (13.1,0);
 \end{tikzpicture}
 \caption{\label{fig:input} Interval representation of an \AMNV instance $\mathcal{I}=(X,D,\dom,N)$, with $X=\set{x_1,\ldots, x_{15}}$, $N=6$, $D=\set{1,\ldots,14}$, and $\dom(x_1)=\set{1,2}, \dom(x_2)=\set{2,3,10}$, etc.}
\end{figure}
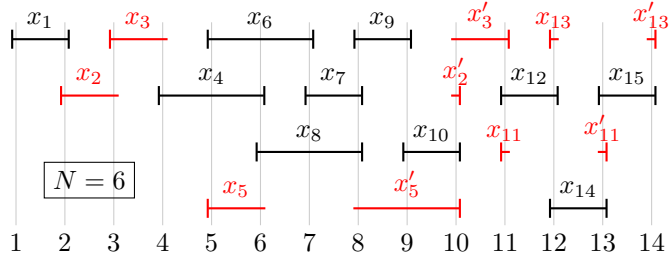

%In the realm of parameterized complexity it is then natural to ask whether \AMNV has a polynomial kernel.
% In other words, is there a preprocessing algorithm returning, in polynomial time, an equivalent instance whose size is bounded by a polynomial function of the parameter?
% Next subsection, we give a linear time kernelization
% algorithm. 

% We then prove its correctness and that the size of the
% produced instance can be bounded by $O(k^2)$. In Subsection
% \ref{subsec:FPTalgo}, we give an \FPT\ algorithm, which uses the
% kernelization algorithm, for checking the consistency of an \AMNV
% constraint in time $O(1.6181^k k^2 +|\mathcal{I}|)$. \HAC can then be
% enforced by applying this algorithm $O(|D|)$ times.

\bigskip

%\subsubsection{Kernelization Algorithm}

Let $\mathcal{I}=(X,D,\dom,N)$ be an instance for the consistency problem for \AMNV constraints. 
The algorithm is more intuitively described using the interval representation of the instance.%
%An interval $I$ is always associated to one variable.
The \emph{friends} of an interval~$I$ are the other intervals of $I$'s variable. An interval is \mydef{optional} if it has at least one friend, and \emph{required} otherwise. For a value $v\in D$, let $\ivl(v)$ denote the set of intervals containing $v$.
%An interval is \mydef{optional} ifis one of at least two intervals of a variable $v$; the \mydef{friends} of this interval are the other intervals of $v$. An interval is \mydef{required} if it is the domain of a variable; a required interval has no friends. 
%The \mydefalt{left endpoint}{left endpoint} $\lep(I)$ and \mydefalt{right endpoint}{right endpoint} $\rep(I)$ of $I$ are the values $\min I$ and $\max I$, respectively.

A \emph{solution} for $\mathcal{I}$ is a subset $S\subseteq D$ of at most $N$ values such that there exists an instantiation assigning the values in $S$ to the variables in $X$. 
The algorithm may detect for some value $v\in D$, that, if the problem has a solution, then it has a solution containing $v$. In this case, the algorithm \emph{selects} $v$, i.e., it removes all variables whose domain contains $v$, it removes $v$ from $D$, and it decrements $N$ by one.
The algorithm may detect for some value $v\in D$, that, if the problem has a solution, then it has a solution not containing $v$. In this case, the algorithm \emph{discards}~$v$, i.e., it removes $v$ from every domain and from $D$. (Note that no new holes are created since $D$ is replaced by $D\setminus \set{v}$.)
The algorithm may detect for some variable $x$, that every solution for $(X\setminus \set{x},D,\dom|_{X\setminus \set{x}},N)$ contains a value from $\dom(x)$. In that case, it \emph{removes} $x$. %: the variable $x$ and its intervals are deleted.

The algorithm sorts the intervals by increasing right endpoint (ties are broken arbitrarily). Then, it exhaustively applies the following three reduction rules.
% described below in order of their appearance. Thus, a reduction rule can only be applied when none of the previous rules applies.

\begin{description}
\item{\RedIncl:} If there are two intervals $I,I'$ such that $I'\subseteq I$ and $I'$ is required, then remove the variable of $I$ (and its intervals).
\item{\RedDom:} If there are two values $v,v'\in D$ such that $\ivl(v') \subseteq \ivl(v)$, then discard $v'$.
\item{\RedUnit:} If $|\dom(x)|=1$ for some variable $x$, then select the value in $\dom(x)$.
%\item{\RedFirst:} If the first interval $I$ is required, then remove $I$ and all intervals that intersect $\rep(I)$, and remove their friends. $N:=N-1$.
\end{description}

%When the previous reduction rule does not apply, the first interval is optional.

\noindent
In the example from Fig.~\ref{fig:input}, \RedIncl removes the variables $x_5$ and $x_8$ because $x_{10}\subseteq x_5'$ and $x_7\subseteq x_8$, \RedDom removes the values $1$ and $5$, \RedUnit selects $2$, which deletes variables $x_1$ and $x_2$, and \RedDom removes $3$ from $D$. The resulting instance is depicted in Fig.~\ref{fig:beforeScanning}.

After none of the previous rules apply, the algorithm scans the
remaining intervals from left to right (i.e., by increasing
  right endpoint). An interval that has already been scanned is either
a \mydef{leader} or a \mydef{follower} of a %non-empty
subset of leaders. Informally, for a leader $L$, if a solution contains
$\rep(L)$, then there is a solution containing $\rep(L)$ and the right
endpoint of each of its
followers. % When a new interval is scanned, an interval that intervals that have already been scanned are \emph{active} or \emph{passive}.

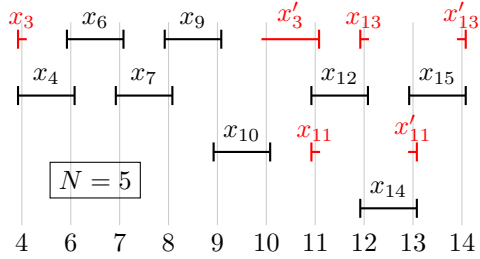
\begin{figure} \centering
 \begin{tikzpicture}[xscale=0.65,yscale=0.75]
  \tikzset{ivl/.style={thick},
  var/.style={midway,above=0.5pt}}
 \tikzset{ivlopt/.style={thick,red}}
 
  \draw (5,-0.6) node {$4$};
  \foreach \x in {6,...,14} \draw (\x,-0.6) node {$\x$};
  \foreach \x in {5,...,14} \draw[thin,black!20] (\x,-0.3) -- (\x,3.3);
  
  \draw (6.5,0.5) node[rectangle,draw] {$N=5$};
 
%  \draw[ivl,|-|] (0.9,3)-- node[var] {$x_1$} (2.1,3);
  \draw[ivlopt,|-] (4.9,3)-- node[var] {$x_3$} (5.1,3);
  \draw[ivl,|-|] (5.9,3)-- node[var] {$x_6$} (7.1,3);
  \draw[ivl,|-|] (7.9,3)-- node[var] {$x_9$} (9.1,3);
  \draw[ivlopt,-|] (9.9,3)-- node[var] {$x_3'$} (11.1,3);
  \draw[ivlopt,|-] (11.9,3)-- node[var] {$x_{13}$} (12.1,3);
  \draw[ivlopt,-|] (13.9,3)-- node[var] {$x_{13}'$} (14.1,3);
%  \draw[ivlopt,|-] (1.9,2)-- node[var] {$x_2$} (3.1,2);
  \draw[ivl,|-|] (4.9,2)-- node[var] {$x_4$} (6.1,2);
  \draw[ivl,|-|] (6.9,2)-- node[var] {$x_7$} (8.1,2);
%  \draw[ivlopt,-|] (9.9,2)-- node[var] {$x_2'$} (10.1,2);
  \draw[ivl,|-|] (10.9,2)-- node[var] {$x_{12}$} (12.1,2);
  \draw[ivl,|-|] (12.9,2)-- node[var] {$x_{15}$} (14.1,2);
%  \draw[ivl,|-|] (5.9,1)-- node[var] {$x_8$} (8.1,1);
  \draw[ivl,|-|] (8.9,1)-- node[var] {$x_{10}$} (10.1,1);
  \draw[ivlopt,|-] (10.9,1)-- node[var] {$x_{11}$} (11.1,1);
  \draw[ivlopt,-|] (12.9,1)-- node[var] {$x_{11}'$} (13.1,1);
%  \draw[ivlopt,|-] (4.9,0)-- node[var] {$x_5$} (6.1,0);
%  \draw[ivlopt,-|] (7.9,0)-- node[var] {$x_5'$} (10.1,0);
  \draw[ivl,|-|] (11.9,0)-- node[var] {$x_{14}$} (13.1,0);
 \end{tikzpicture}
 \vspace{-1pt}
 \caption{\label{fig:beforeScanning} Instance obtained from the instance of Fig.~\ref{fig:input} by exhaustively applying rules \RedIncl, \RedDom, and \RedUnit.}
\end{figure}

The algorithm scans the first intervals up to, and including, the first required interval. All these intervals become leaders.

The algorithm then continues scanning intervals one by one. Let $I$ be the interval that is currently scanned and $I_{p}$ be the last interval that was scanned. The \mydef{active} intervals are those that have already been scanned and intersect $I_{p}$. A \mydef{popular} leader is a leader that is either active or has at least one active follower.
\begin{itemize}
% \item If all active intervals intersect $I$, then \textbf{remove} $I$'s variable.
% The previous case never applies
 \item If $I$ is optional, then $I$ becomes a leader, the algorithm continues scanning intervals until scanning a required interval; all these intervals become leaders.
 \item If $I$ is required, then it becomes a follower of all popular leaders that do not intersect $I$ and that have no follower intersecting $I$. If all popular leaders have at least two followers, then set $N:=N-1$ and \textbf{merge} the second-last follower of each popular leader with the last follower of the corresponding leader; i.e., for every popular leader, the right endpoint of its second-last follower is set to the right endpoint of its last follower, and then the last follower of every popular leader is removed.
\end{itemize}

\noindent
After having scanned all the intervals, the algorithm exhaustively applies the reduction rules \RedIncl, \RedDom, and \RedUnit again.
%\begin{description}
%\item{\RedEmptyDomain:} If there is a variable $v\in V$ such that $\dom(v)=\emptyset$, then remove $v$.
%\item{\RedDoubleVal:} If there are two values that occur in the domains of the same variables, then remove one of these values from $D$ and update the map $\dom$ accordingly.
%\end{description}

In the example from Fig.~\ref{fig:beforeScanning}, the interval of variable $x_6$ is merged with $x_9$'s interval, and the interval of $x_7$ with the interval of $x_{10}$. \RedDom then removes the values~$7$ and $8$, resulting in the instance depicted in Fig.~\ref{fig:kernel}.

The correctness and performance guarantee of this kernelization algorithm
are proved in \ref{app:kernel}. In particular, for the correctness,
we prove that a solution $S$ for an instance $\mathcal{I}$ can be obtained from a solution $S'$ for an instance $\mathcal{I}'$ that is obtained from $\mathcal{I}$ by one \textbf{merge}-operation by adding to $S'$
one value that is common to all second-last followers of the popular leaders that were merged.
We can easily bound the number of leaders by $4k$ and we prove that each leader has at most $4k$ followers.
Since each interval is a leader or a follower of at least one leader, this bounds the total number
of intervals by $O(k^2)$.
Using the succinct description of the domains, the size of the kernel is $O(k^2)$.
We also give some details for a linear-time implementation of the algorithm.

\smallskip

\noindent
\emph{Remark:} Denoting $\textsf{var}(v)=\set{x\in X: v\in \dom(x)}$, Rule \RedDom can be generalized to discard any $v'\in D$ for which there exists a $v\in D$ such that $\textsf{var}(v')\subseteq \textsf{var}(v)$ at the expense of a higher running time.

\medskip
\noindent
The kernel for \AMNV-\PR{Cons}{\holes} can now be used to derive a kernel for \textsc{NValue}-\PR{Cons}{\holes}.

\begin{corollary}\label{cor:nvalue}
 The problem \textsc{NValue}-$\PR{Cons}{\holes}$ has a polynomial
  kernel. In particular, an \textsc{NValue} constraint with $k$ holes can be
  reduced in $O((|X|+|D|)^{\omega/2})$ time to a consistency-equivalent \textsc{NValue} constraint
  with $O(k^2)$ variables and $O(k^2)$ domain values, where $\omega<2.3729$ is the exponent of matrix multiplication.
\end{corollary}
\begin{proof}
 As in \cite{BessiereEtal06}, we determine the largest possible value for $N$ if its domain were the set of all integers.
 This can be done in $O((|X|+|D|)^{\omega/2})$ time \cite{MuchaS04,VassilevskaWilliams12} by computing a maximum matching in the graph whose vertices are $X \cup D$ with an edge between
 $x\in X$ and $v\in D$ iff $v\in \dom(x)$.
 Suppose this largest possible value is $N^+$.
 Now, set $\dom(N):=\set{v\in \dom(N) : v\le N^+}$, giving a consistency-equivalent \textsc{NValue} constraint.
 Note that if this constraint has a legal instantiation $\alpha$ with $\alpha(N)\le \max(\dom(N))$, then
 it has a legal instantiation $\alpha'$ with $\alpha'(N)= \max(\dom(N))$.
 Therefore, it suffices to compute a kernel for \AMNV-$\PR{Cons}{\holes}$ with the same variables and domains and value $N = \max(\dom(N))$,
 and return it.
\end{proof}

\subsection{Improved FPT Algorithm using the Kernel}
\label{subsec:FPTalgo}

Using the kernel from Theorem \ref{thm:kernel} and the simple algorithm
described in the beginning of this section, one arrives at a $O(2^k k^2
+ |\mathcal{I}|)$ time algorithm for checking the consistency of an
\AMNV constraint. Borrowing ideas from the kernelization algorithm, we
now reduce the exponential dependency on $k$ in the running time.  The
speed-ups due to this branching algorithm and the kernelization
algorithm lead to a speed-up for enforcing \HAC for \AMNV constraints
(by Corollary \ref{cor:HAC}) and for enforcing \HAC for \NV constraints
(by the decomposition of \cite{BessiereEtal06}).

\begin{theorem}\label{thm:FPTalgo}
The Consistency problem for \AMNV constraints admits a $O(\phi^k k^2+|\mathcal{I}|)$ time algorithm, where $k$ is the number of holes in the domains of the input instance $\mathcal{I}$, and $\phi=\frac{1+\sqrt{5}}{2}<1.6181$.
\end{theorem}
\begin{proof}
The first step of the algorithm invokes the kernelization algorithm and obtains an equivalent instance $\mathcal{I}'$ with $O(k^2)$ intervals in time $O(|\mathcal{I}|)$.

Now, we describe a branching algorithm checking the consistency of
$\mathcal{I}'$. Let $I_1$ denote the first interval of~$\mathcal{I}'$
(in the ordering by increasing right endpoint). $I_1$ is optional. Let
$\mathcal{I}_1$ denote the instance obtained from $\mathcal{I}'$ by
selecting $\rep(I_1)$ and exhaustively applying Reduction Rules \RedDom
and \RedUnit. Let $\mathcal{I}_2$ denote the instance obtained from
$\mathcal{I}'$ by removing $I_1$ (if $I_1$ had exactly one friend, this
friend becomes required) and exhaustively applying Reduction Rules
\RedDom and \RedUnit. Clearly, $\mathcal{I}'$ is consistent \myIff
$\mathcal{I}_1$ or $\mathcal{I}_2$ is consistent.

Note that both $\mathcal{I}_1$ and $\mathcal{I}_2$ have at most $k-1$ holes. If either $\mathcal{I}_1$ or $\mathcal{I}_2$ has at most $k-2$ holes, the algorithm recursively checks whether at least one of $\mathcal{I}_1$ and $\mathcal{I}_2$ is consistent.
If both $\mathcal{I}_1$ and $\mathcal{I}_2$ have exactly $k-1$ holes, we note that in $\mathcal{I}'$,
\begin{enumerate}
\item $I_1$ has one friend,
\item no other optional interval intersects $I_1$, and
\item the first interval of both $\mathcal{I}_1$ and $\mathcal{I}_2$ is $I_f$, which is the third optional interval in $\mathcal{I}'$ if the second optional interval is the friend of $I_1$, and the second optional interval otherwise.
\end{enumerate}
Thus, the instance obtained from $\mathcal{I}_1$ by removing $I_1$'s friend and applying \RedDom and \RedUnit may differ from $\mathcal{I}_2$ only in $N$. Let $s_1$ and $s_2$ denote the number of values smaller than $\rep(I_f)$ that have been selected to obtain $\mathcal{I}_1$ and $\mathcal{I}_2$ from $\mathcal{I}'$, respectively. If $s_1 \le s_2$, then the non-consistency of $\mathcal{I}_1$ implies the non-consistency of $\mathcal{I}_2$. Thus, the algorithm need only recursively check whether $\mathcal{I}_1$ is consistent. On the other hand, if $s_1 > s_2$, then the non-consistency of $\mathcal{I}_2$ implies the non-consistency of $\mathcal{I}_1$. Thus, the algorithm need only recursively check whether $\mathcal{I}_2$ is consistent.

The recursive calls of the algorithm may be represented by a search tree labeled with the number of holes of the instance. As the algorithm either branches into only one subproblem with at most $k-1$ holes, or two subproblems with at most $k-1$ and at most $k-2$ holes, respectively, the number of leaves of this search tree is
%\begin{align*}
$
T(k) \le T(k-1)+T(k-2),
$
%\end{align*}
with $T(0)=T(1)=1$.
Using standard techniques in the analysis of exponential time algorithms (see, e.g., \cite[Chapter 2]{FominKratsch10} and \cite[Lemma 2.3]{Gaspers10}),
it suffices to find a value $c>1$ for the base of the exponential function bounding the running time, that we will minimize, such that
\begin{align*}
c^{k-1} + c^{k-2} &\le c^k & \text{for all } k\ge 0,
\intertext{or, equivalently, such that}
c+1 \le c^2
\end{align*}
It now suffices to find the unique positive real root of $x^2-x-1$, which is $\phi=\frac{1+\sqrt{5}}{2}<1.6181$, to determine the optimal value of $c$ for this analysis.

Since the size of the search tree is $O(\phi^k)$ and the number of operations executed at each node of the search tree is $O(k^2)$, the running time of the branching algorithm can be upper bounded by $O(\phi^k k^2)$.
\end{proof}

\noindent
For the example of Fig.~\ref{fig:kernel}, the instances $\mathcal{I}_1$
and $\mathcal{I}_2$ are computed by selecting the value $4$, and
removing the interval~$x_3$, respectively. The reduction rules select
the value $9$ for $\mathcal{I}_1$ and the values $6$ and $10$ for
$\mathcal{I}_2$. Both instances start with the interval $x_{11}$, and
the algorithm recursively solves $\mathcal{I}_1$ only, where
the values $12$ and $13$ are selected, leading to the
solution $\set{4,9,12,13}$  for the kernelized
instance, which corresponds to the solution $\set{2,4,7,9,12,13}$ for
the instance of Fig.~\ref{fig:input}.

%\smallskip
\begin{corollary}\label{cor:HAC}
\HAC for an \AMNV constraint can be enforced in time $O(\phi^k \cdot k^2 \cdot |D|+|\mathcal{I}|\cdot |D|)$, where $k$ is the number of holes in the domains of the input instance $\mathcal{I}=(X,D,\dom,N)$, and $\phi=\frac{1+\sqrt{5}}{2}<1.6181$.
\end{corollary}
\begin{proof}
We first remark that if a value $v$ can be filtered from the domain of a variable $x$ (i.e., $v$ has no support for $x$), then $v$ can be filtered from the domain of all variables, as for any legal instantiation $\alpha$ with $\alpha(x')=v$, $x'\in X\setminus \set{x}$, the assignment obtained from $\alpha$ by setting $\alpha(x):=v$ is a legal instantiation as well. % as the number of distinct values does not increase. 
Also, filtering the value $v$ creates no new holes as the set of values can be set to $D\setminus \set{v}$.

Now we enforce \HAC by applying $O(|D|)$ times
%$\min(n,|D|)$ times.
the algorithm from Theorem \ref{thm:FPTalgo}.
Assume the instance $\mathcal{I}=(X,D,\dom,N)$ is consistent. If $(X,D,\dom,N-1)$ is consistent, then no value can be filtered.
Otherwise, check, for each $v\in D$, whether the instance obtained from selecting $v$ is consistent and filter $v$ if this is not the case.
%If $n<|D|$, then check, for some $x\in X$, whether $(X\setminus \set{x},D,\dom|_{X\setminus \set{x}},N-1)$ is consistent. If so, no value can be filtered from $\dom(x)$; mark all values in $\dom(x)$ and recursively filter the instance $(X\setminus \set{x},D,\dom|_{X\setminus \set{x}},N)$. If not, then filter all unmarked values that occur only in the domain of $x$; recursively filter the ... does not seem to work
\end{proof}

Using the same reasoning as in Corollary \ref{cor:nvalue}, we now obtain the following corollary for \textsc{NValue}.

\begin{corollary}
\HAC for an \textsc{NValue} constraint can be enforced in time $O((\phi^k \cdot k^2 +(|X|+|D|)^{\omega/2})\cdot |D|)$, where $k$ is the number of holes in the domains of the input instance $\mathcal{I}=(X,D,\dom,N)$, $\phi=\frac{1+\sqrt{5}}{2}<1.6181$, and $\omega<2.3729$ is the exponent of matrix multiplication.
\end{corollary}

\section{Bayesian Reasoning} 

\emph{Bayesian networks} (BNs) have emerged as a general
representation scheme for uncertain knowledge \cite{Pearl10}.  A BN
models a set of stochastic variables, the independencies among these
variables, and a joint probability distribution over these variables.
For simplicity we consider the important special case where the
stochastic variables are Boolean. The variables and independencies are
modeled in the BN by a directed acyclic graph $G = (V,A)$, the joint
probability distribution is given by a table $T_v$ for each node $v
\in V$ which defines a probability $T_{v|U}$ for each possible
instantiation $U=(d_1,\dots,d_s)\in \{\True,\False\}^s$ of the parents
$v_1,\dots,v_s$ of $v$ in $G$. The probability $\Prob(U)$ of a
complete instantiation $U$ of the variables of~$G$ is given by the
product of $T_{v|U}$ over all variables~$v$.  We consider the
problem
\textsc{\bfseries Positive-BN-Inference} which takes as input a
Boolean BN $(G,T)$ and a variable $v$, and asks whether
$\Prob(v=\True) > 0$.  
The problem is $\NP$-complete \cite{Cooper90}
and moves from NP to \#P if we ask to compute $\Prob(v=\text{true})$
\cite{Roth96}.  The problem can be solved in polynomial time if the BN
is \emph{singly connected}, i.e, if there is at most one undirected
path between any two variables \cite{Pearl88}. It is natural to
parametrize the problem by the number of variables one must delete in
order to make the BN singly connected (the deleted variables form a
\emph{loop cutset}). This yields the following parameterized problem.
\begin{quote}
  $\PR{Positive-BN-Inference}{loop cutset size}$

  \emph{Instance:} A Boolean BN $(G,T)$, a variable $v$, and a loop
  cutset $S$ of size $k$.

  \emph{Parameter:} The integer $k$.

  \emph{Question:} Is $\Prob(v=\True) > 0$?  

\end{quote}
Again we also state a related recognition problem.
\begin{quote}
  $\RPR{Positive-BN-Inference}{loop cutset size}$

  \emph{Instance:} A Boolean BN $(G,T)$ and an integer $k\geq 0$.

  \emph{Parameter:} The integer $k$.

  \emph{Question:} Does $(G,T)$ has a loop cutset of size $\leq k$?.
\end{quote}

 Now,
$\PR{Positive-BN-Inference}{loop cutset size}$ is easily seen to be
fixed-parameter tractable as we can determine whether
$\Prob(v=\True)>0$ by taking the maximum of $\Prob(v=\True \mid U)$
over all $2^k$ possible instantiations of the $k$ cutset variables,
each of which requires processing of a singly connected network.
However, although fixed-parameter tractable, it is unlikely that the
problem admits a polynomial kernel.
\begin{theorem}
  $\PR{Positive-BN-Inference}{loop cutset size}$ does not admit a polynomial
  kernel unless $\NP \subseteq \coNP/\text{\normalfont poly}$.
\end{theorem} 
\begin{proof}  We give a polynomial parameter transformation
  from $\PR{SAT}{vars}$ and apply Theorem~\ref{the:trans}.  The
  reduction is based on the reduction from 3SAT given by
  Cooper~\cite{Cooper90}.   Let $F$ be a CNF formula on $n$ variables.  We
  construct a BN $(G,T)$ such that for a variable $v$ we have
  $\Prob(v=\True) > 0$ if and only if $F$ is satisfiable. Cooper uses
  \emph{input nodes} $u_i$ for representing variables of~$F$,
  \emph{clause nodes} $c_i$ for representing the clauses of $F$, and
  \emph{conjunction nodes} $d_i$ for representing the conjunction of the
  clauses. For instance, if $F$ has three clauses and four variables,
  then Cooper's reduction produces a BN $(G,T)$ where $G$ has the
  following shape:
  
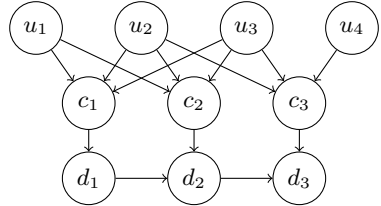
\begin{center}
  \begin{tikzpicture}[xscale=1.4,yscale=1.0]
    \small
    \tikzstyle{round}=[circle,draw,inner sep=1pt, minimum size=7mm]
    \draw 
    (0,0)   node[round] (u1) {$u_1$}
    (1,0)   node[round] (u2) {$u_2$}
    (2,0)   node[round] (u3) {$u_3$}
    (3,0)   node[round] (u4) {$u_4$}
    (.5,-1)   node[round] (c1) {$c_1$}
   (1.5,-1)   node[round] (c2) {$c_2$}
   (2.5,-1)   node[round] (c3) {$c_3$}
    (.5,-2)   node[round] (d1) {$d_1$}
   (1.5,-2)   node[round] (d2) {$d_2$}
   (2.5,-2)   node[round] (d3) {$d_3$}
;
\draw[->] (u1)--(c1);
\draw[->] (u2)--(c1);
\draw[->] (u3)--(c1);
\draw[->] (u1)--(c2);
\draw[->] (u2)--(c2);
\draw[->] (u3)--(c2);
\draw[->] (u2)--(c3);
\draw[->] (u3)--(c3);
\draw[->] (u4)--(c3);
\draw[->] (c1)--(d1);
\draw[->] (c2)--(d2);
\draw[->] (c3)--(d3);
\draw[->] (d1)--(d2);
\draw[->] (d2)--(d3);
 \end{tikzpicture}%
    
  \end{center}
  Clearly, the input nodes form a loop cutset of $G$.  However, in
  order to get a polynomial parameter transformation from
  $\PR{SAT}{vars}$ we must allow in $F$ that clauses contain an
  arbitrary number of literals, not just three.  If we apply Cooper's
  reduction directly, then a single clause node $c_i$ with many
  parents requires a table $T_{c_i}$ of exponential size. To overcome
  this difficulty we simply split clause nodes $c_i$ containing more
  than 3 literals into several clause nodes, as indicated below, where
  the last one feeds into a conjunction node $d_i$.
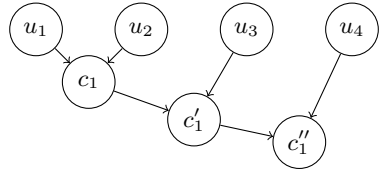
  \begin{center}
  \begin{tikzpicture}[xscale=1.4,yscale=1.0]
    \small
    \tikzstyle{round}=[circle,draw,inner sep=1pt, minimum size=7mm]
    \draw
    (0,0)   node[round] (u1) {$u_1$}
    (1,0)   node[round] (u2) {$u_2$}
    (2,0)   node[round] (u3) {$u_3$}
    (3,0)   node[round] (u4) {$u_4$}
    (.5,-.7)   node[round] (c1) {$c_1$}
    (1.5,-1.2)   node[round] (c2) {$c_1'$}
    (2.5,-1.5)   node[round] (c3) {$c_1''$}
    ;
    \draw[->] (u1)--(c1);
    \draw[->] (u2)--(c1);
    \draw[->] (c1)--(c2);
    \draw[->] (u3)--(c2);
    \draw[->] (c2)--(c3);
    \draw[->] (u4)--(c3);
  \end{tikzpicture}%
\end{center}
\sloppypar \noindent It remains to observe that the set of input nodes
  $E=\{u_1,\dots,u_n\}$ still form a loop cutset of the constructed BN,
  hence we have indeed a polynomial parameter transformation from
  $\PR{SAT}{vars}$ to $\PR{Positive-BN-Inference}{loop cutset
    size}$. The result follows by Theorem~\ref{the:trans}.

  \comment{Perhaps add more details to the proof} \comment{Serge: Yes,
    I think we need more details here. I'll do that.}

\end{proof}

Let us now turn to the recognition problem
$\RPR{Positive-BN-Inference}{loop cutset size}$.

\begin{proposition}\label{prop:kern-BN}
  $\RPR{Positive-BN-Inference}{loop cutset size}$ admits a polynomial
  kernel with $O(k^2)$ nodes.
\end{proposition}
\begin{proof}
  Let $((G,T),k)$ be an instance of $\RPR{Positive-BN-Inference}{loop
    cutset size}$. We note that loop cutsets of $(G,T)$ are just the
  so-called \emph{feedback vertex sets} of $G$. Hence we can apply a
  known kernelization algorithm for feedback vertex sets
  \cite{CaoChenLiu10} to $G$ and obtain a kernel $(G',k)$ with at most
  $O(k^2)$ many vertices. We translate this into an
  instance $(G',T',k')$ of $\RPR{Positive-BN\--Inference}{loop
    cutset size}$ by taking an arbitrary table $T'$. 
\end{proof}

% Finding a loop cutset of size at most $k$, if one
% exists, is \fpt \cite{ChenLiuLuOsullivanRazgon08}.

\section{Nonmonotonic Reasoning}
\label{sec:nonmonotonic}

\emph{Logic programming with negation} under the stable model semantics
is a well-studied form of nonmonotonic reasoning
\cite{GelfondLifschitz88,MarekTruszczynski99}.  A (normal) \emph{logic
  program}~$P$ is a finite set of rules $r$ of the form
\[h \longleftarrow a_1 \wedge \dots \wedge a_m \wedge \neg b_1 \wedge
\dots \wedge \neg b_n\] where $h,a_i,b_i$ are \emph{atoms}, where $h$
forms the head and the $a_i,b_i$ from the body of $r$.  We write
$H(r)=h$, $B^+(r)=\{a_1,\dots,a_m\}$, and
$B^-(r)=\{b_1,\dots,b_n\}$. Let~$I$ be a finite set of atoms.  The
\emph{GF reduct} $P^I$ of a logic program~$P$ under $I$ is the program
obtained from~$P$ by removing all rules $r$ with $B^-(r)\cap I\neq
\emptyset$, and removing from the body of each remaining rule $r'$ all
literals $\neg b$ with $b\in I$.  $I$ is a \emph{stable model} of $P$
if $I$ is a minimal model of $P^I$, i.e., if (i) for each rule $r\in
P^I$ with $B^+(r)\subseteq I$ we have $H(r)\in I$, and (ii) there is
no proper subset of $I$ with this property. The \emph{undirected
  dependency graph} $U(P)$ of $P$ is formed as follows. We take the
atoms of $P$ as vertices and add an edge $x-y$ between two atoms $x,y$
if there is a rule $r\in P$ with $H(r)=x$ and $y\in B^+(r)$, and we
add a path $x-u-y$ if $H(r)=x$ and $y\in B^-(r)$ ($u$ is a new vertex
of degree 2). The \emph{feedback width} of $P$ is the size of a
smallest set $V$ of atoms such that every cycle of $U(P)$ runs through
an atom in $V$ (such a set $V$ is called a \emph{feedback vertex set}).

A fundamental computational problems is \textbf{Stable Model Existence
  (SME)}, which asks whether a given normal logic program has a stable
model.  The problem is well-known to be
$\NP$-complete~\cite{MarekTruszczynski91}.  Gottlob
\etal~\cite{GottlobScarcelloSideri02} considered the following
parameterization of the problem and showed fixed-parameter
tractability (see \cite{FichteSzeider11} for
generalizations).

\begin{quote}
$\PR{SME}{feedback width}$

  \emph{Instance:} A logic program $P$ and feedback vertex set $V$ of
  size $k$.

  \emph{Parameter:} The integer $k$.

  \emph{Question:} Does $P$ have a stable model?

\end{quote}
Again we also state a related recognition problem.
\begin{quote}
$\RPR{SME}{feedback width}$

  \emph{Instance:} A logic program $P$ and an integer $k \geq 0$.

  \emph{Parameter:} The integer $k$.

  \emph{Question:} Does $P$ have a a feedback vertex set of size at
  most $k$? 
\end{quote}
We show that the result of Gottlob
\etal~\cite{GottlobScarcelloSideri02} cannot be strengthened towards a
polynomial kernel.
\begin{theorem}
  $\PR{SME}{feedback width}$ does not admit a polynomial kernel unless
  $\NP \subseteq \coNP/\text{\normalfont poly}$.
\end{theorem}
\begin{proof} We give a polynomial parameter transformation from
  $\PR{SAT}{vars}$ to $\textbf{SME}($feedback width$)$ using a
  construction of Niemel{\"a}~\cite{Niemela99}. Given a CNF formula
  $F$ on $n$ variables, we construct a logic program $P$ as
  follows. For each variable $x$ of $F$ we take two atoms $x$ and
  $\hat x$ and include the rules $(\hat x \leftarrow \neg x)$ and $(x
  \leftarrow \neg \hat x)$; for each clause $C$ of $F$ we take an atom
  $c$ and include for each positive literal $a$ of $C$ the rule $(c
  \leftarrow a)$, and for each negative literal $\neg a$ of $C$ the
  rule $(c \leftarrow \hat a)$; finally, we take two atoms $s$ and $f$
  and include the rule $(f \leftarrow \neg f \wedge \neg s)$ and for
  each clause $C$ of $F$ the rule 
  % fixed typo in next line
  $(s \leftarrow \neg c)$.  The formula $F$ is satisfiable if and only
  if $P$ has a stable model~\cite{Niemela99}. It remains to observe
  that each cycle of $U(P)$ runs through a vertex in $V=\SB x,\hat x
  \SM x \in \text{vars}(F)\SE$, hence the feedback width of $P$ is at
  most $2n$.  Hence we have a polynomial parameter transformation from
  $\PR{SAT}{vars}$ to $\PR{SME}{feedback width}$. The result follows
  by Theorem~\ref{the:trans}.
\end{proof}

Using a similar approach as for Proposition~\ref{prop:kern-BN} we can
establish the following result.

\begin{proposition}
  $\RPR{SME}{feedback width}$ admits a polynomial kernel with $O(k^2)$
  atoms.
\end{proposition}

\section{Conclusion}
\label{sec:conclusion}

We have provided the first theoretical evaluation of the guarantees and
limits of polynomial-time preprocessing for hard AI problems. In
particular we have established super-polynomial kernel lower bounds
for many problems, providing firm limitations for the power of
polynomial-time preprocessing for these problems.  On the positive
side, we have developed an efficient linear-time kernelization
algorithm for the consistency problem for \AMNV constraints, and have
shown how it can be used to speed up the complete propagation of \NV
and related constraints.

Subsequent to our work, Fellows
\etal~\cite{FellowsPfandlerRosamondRummele12} investigated the
parameterized complexity and kernelization for various
parameterizations of Abductive Reasoning.  Their kernelization results
were mostly negative, showing that many parameterizations for the
Abduction problem have no polynomial kernels unless $\NP \subseteq
\coNP/\text{\normalfont poly}$. Similarly negative are the
kernelization results of B\"{a}ckstr\"{o}m
\etal~\cite{BackstromJonssonOrdyniakSzeider13} for planning problems,
parameterized by the length of the plan.

%On the negative side, we have established a theoretical
%result which indicates that \EGC constraints do not admit
%polynomial kernels. 

We conclude from these results that in contrast to many optimization
problems (see Section 1), typical AI problems do not admit polynomial
kernels.  Our results suggest the consideration of alternative
approaches.  For example, it might still be possible that some of the
considered problems admit polynomially sized Turing kernels, i.e., a
polynomial-time preprocessing to a Boolean combination of a polynomial
number of polynomial kernels. In the area of optimization, parameterized
problems are known that do not admit polynomial kernels but admit
polynomial Turing kernels~\cite{FernauEal09}.  This suggests a
theoretical and empirical study of Turing kernels for the AI problems
considered.

\appendix

\section{Appendix: Proof of Theorem~\ref{thm:kernel}}
\label{app:kernel}

In this appendix, we prove Theorem~\ref{thm:kernel} by proving the correctness of the algorithm, upper bounding the size of the kernel, and analyzing its running time.

Let $\mathcal{I}'=(X',D',\dom',N')$ be the instance resulting from applying one operation of the kernelization algorithm to an instance $\mathcal{I}=(X,D,\dom,N)$. An operation is an instruction which modifies the instance: \RedIncl, \RedDom, \RedUnit, and \textbf{merge}. %; all operations of the kernelization algorithm are typeset in boldface.
We show that there exists a solution $S$ for $\mathcal{I}$ \myIff there exists a solution $S'$ for $\mathcal{I'}$. A solution is \emph{nice} if each of its elements is the right endpoint of some interval. Clearly, for every solution, a nice solution of the same size can be obtained by shifting each value to the next right endpoint of an interval. Thus, when we construct $S'$ from $S$ (or vice-versa), we may assume that $S$ is nice.

Reduction Rule \RedIncl is sound because a solution for $\mathcal{I}$ is a solution for $\mathcal{I'}$ and vice-versa, because any solution $\mathcal{I}'$ contains a value $v$ of $I\subseteq I'$, as $I$ is required.
Reduction Rule \RedDom is correct because if $v'\in S$, then $S':=(S\setminus \set{v'})\cup \set{v}$ is a solution for $\mathcal{I}'$ and for $\mathcal{I}$.
Reduction Rule \RedUnit is obviously correct ($S=S'\cup \dom(x)$).

%For Reduction Rule \RedFirst, observe that $X$ contains some value of $I$. As $I$ is the first interval in the ordering by increasing right endpoint, every interval that intersects $I$, also intersects $\rep(I)$. Thus, the greedy choice of assigning $\rep(I)$ to every variable whose domain contains $\rep(I)$ is optimal. The algorithm thus removes all intervals that intersect $\rep(I)$ and their friends. This creates variables with empty domains, which are later removed by Reduction Rule \RedEmptyDomain.

After having applied these $3$ reduction rules, observe that the first interval is optional and contains only one value.
Suppose the algorithm has started scanning intervals.
By construction, the following properties apply to $\mathcal{I}'$.

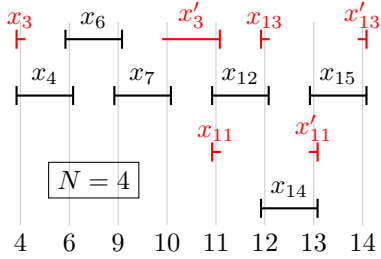
\begin{figure} \centering
\begin{tikzpicture}[xscale=0.65,yscale=0.75]
  \tikzset{ivl/.style={thick},
  var/.style={midway,above=0.5pt}}
\tikzset{ivlopt/.style={thick,red}}

  \draw (7,-0.6) node {$4$};
  \draw (8,-0.6) node {$6$};
  \foreach \x in {9,...,14} \draw (\x,-0.6) node {$\x$};
  \foreach \x in {7,...,14} \draw[thin,black!20] (\x,-0.3) -- (\x,3.3);
  
  \draw (8.5,0.5) node[rectangle,draw] {$N=4$};
 
%  \draw[ivl,|-|] (0.9,3)-- node[var] {$x_1$} (2.1,3);
  \draw[ivlopt,|-] (6.9,3)-- node[var] {$x_3$} (7.1,3);
  \draw[ivl,|-|] (7.9,3)-- node[var] {$x_6$} (9.1,3);
  \draw[ivlopt,-|] (9.9,3)-- node[var] {$x_3'$} (11.1,3);
  \draw[ivlopt,|-] (11.9,3)-- node[var] {$x_{13}$} (12.1,3);
  \draw[ivlopt,-|] (13.9,3)-- node[var] {$x_{13}'$} (14.1,3);
%  \draw[ivlopt,|-] (1.9,2)-- node[var] {$x_2$} (3.1,2);
  \draw[ivl,|-|] (6.9,2)-- node[var] {$x_4$} (8.1,2);
  \draw[ivl,|-|] (8.9,2)-- node[var] {$x_7$} (10.1,2);
%  \draw[ivlopt,-|] (9.9,2)-- node[var] {$x_2'$} (10.1,2);
  \draw[ivl,|-|] (10.9,2)-- node[var] {$x_{12}$} (12.1,2);
  \draw[ivl,|-|] (12.9,2)-- node[var] {$x_{15}$} (14.1,2);
%  \draw[ivl,|-|] (5.9,1)-- node[var] {$x_8$} (8.1,1);
%  \draw[ivl,|-|] (8.9,1)-- node[var] {$x_{10}$} (10.1,1);
  \draw[ivlopt,|-] (10.9,1)-- node[var] {$x_{11}$} (11.1,1);
  \draw[ivlopt,-|] (12.9,1)-- node[var] {$x_{11}'$} (13.1,1);
%  \draw[ivlopt,|-] (4.9,0)-- node[var] {$x_5$} (6.1,0);
%  \draw[ivlopt,-|] (7.9,0)-- node[var] {$x_5'$} (10.1,0);
  \draw[ivl,|-|] (11.9,0)-- node[var] {$x_{14}$} (13.1,0);
 \end{tikzpicture}
 \caption{\label{fig:kernel} Kernelized instance.}
\end{figure}

\begin{property}\label{obs:follower-not-intersect-leader}
 A follower does not intersect any of its leaders.
\end{property}

\begin{property}\label{obs:followers-no-intersect}
 If $I,I'$ are two distinct followers of the same leader, then $I$ and $I'$ do not intersect.
\end{property}

%\noindent
%The following claim implies the correctness of the \textbf{remove} operation.
%
%\begin{boldclaim}
% Let $A$ denote the set of active intervals when a new interval $I$ is scanned. Suppose the previously scanned interval was required. Any nice solution for $\mathcal{I}$ contains the right endpoint of some interval of $A$.
%\end{boldclaim}
%\begin{proof}
%Let $L_p$ denote the previously scanned interval that was not removed by the algorithm. The interval $L_p$ is required, as every item of the scanning phase of the algorithm ends by scanning a required interval. Thus, any nice solution intersects $I_p$ in the right endpoint of some interval of $A$.
%\end{proof}

\noindent
Before proving the correctness of the \textbf{merge} operation, let us first show that the subset of leaders of a follower is not empty.
\begin{boldclaim}\label{cl:atleastoneleader}
 Every interval that has been scanned is either a leader or a follower of at least one leader.
\end{boldclaim}
\begin{proof}
First, note that \RedDom ensures that each domain value in $D$ is the left endpoint of some interval and the right endpoint of some interval.
We show that when an interval $I$ is scanned it either becomes a leader or a follower of at least one leader.
By induction, assume this is the case for all previously scanned intervals.
Denote by $I_p$ the interval that was scanned prior to $I$.
If $I_p$ or $I$ is optional, then $I$ becomes a leader.
Suppose $I$ and $I_p$ are required.
We have that $\lep(I)>\lep(I_p)$, otherwise $I$ would have been removed by \RedIncl.
By Rule \RedDom, there is some interval $I_\ell$ with $\rep(I_\ell)=\lep(I_p)$.
If $I_\ell$ is a leader, $I$ becomes a follower of $I_\ell$; otherwise $I$ becomes a follower of $I_\ell$'s leader.
\end{proof}

\noindent
We will now prove the correctness of the \textbf{merge} operation.
Recall that $\mathcal{I}'$ is an instance obtained from $\mathcal{I}$ by one application of the \textbf{merge} operation.
Let $I$ denote the interval that is scanned when the \textbf{merge} operation is applied.
At this computation step, each popular leader has at least
two followers and the algorithm merges the last two followers of each
popular leader and decrements $N$ by one.
Let $F_2$ denote the set of all  intervals that are the second-last follower of a popular leader, and $F_1$ the set of all intervals that are the last follower of a popular leader before merging. Let $M$ denote the set of merged intervals. Clearly, every interval of $F_1\cup F_2\cup M$ is required as all followers are required.

\begin{lemma}\label{lem:lastfollower-commonvalue}
 Every interval in $F_1$ intersects $\lep(I)$.
\end{lemma}
\begin{proof}
Let $I_1\in F_1$.
By construction, $\rep(I_1)\in I$, as $I$ becomes a follower of every popular leader that has no follower intersecting $I$, and no follower has a right endpoint larger than $\rep(I)$. Moreover, $\lep(I_1)\le \lep(I)$ as no follower is a strict subset of $I$ by \RedIncl and the fact that all followers are required.
\end{proof}

The correctness of the \textbf{merge} operation will follow from the next two lemmas.

\begin{lemma}
If $S$ is a nice solution for $\mathcal{I}$, then there exists a solution $S'$ for $\mathcal{I}'$ with $S'\subseteq S$.
\end{lemma}
\begin{proof}
Let $I^-$ be the interval of $F_2$ with the largest right endpoint. Let $L$ be a leader of $I^-$. By construction and \RedIncl, $L$ is a leader of $I$ as well and is therefore popular. Let $t_1\in S\cap I$ be the smallest value of $S$ that intersects $I$ and let $t_2 \in S \cap I^-$ be the largest value of $S$ that intersects $I^-$. By Property \ref{obs:followers-no-intersect}, we have that $t_2<t_1$.

\begin{boldclaim}\label{cl:no-intermediate-value}
 The set $S$ contains no value $t_0$ such that $t_2<t_0<t_1$.
\end{boldclaim}
\begin{proof}
For the sake of contradiction, suppose $S$ contains a value $t_0$ such that $t_2<t_0<t_1$.
Since $S$ is nice, $t_0$ is the right endpoint of some interval $I_0$.
Since $t_2$ is the rightmost value intersecting $S$ and any interval in $F_2$, the interval $I_0$ is not in $F_2$.
Since $I_0$ has already been scanned, and was scanned after every interval in $F_2$, the interval $I_0$ is in $F_1$.
However, by Lemma \ref{lem:lastfollower-commonvalue}, $I_0$ intersects $\lep(I)$.
Since no scanned interval has a larger right endpoint than $I$, we have that $t_0 \in S\cap I$, which contradicts the fact that $t_1$ is the smallest value in $S\cap I$ and that $t_0<t_1$.
\end{proof}

\begin{boldclaim}\label{cl:t1intersectsmerged}
 Suppose $I_1\in F_1$ and $I_2\in F_2$ are the last and second-last follower of a popular leader $L'$, respectively. Let $M_{12}\in M$ denote the interval obtained from merging $I_2$ with $I_1$. If $t_2\in I_2$, then $t_1\in M_{12}$.
\end{boldclaim}
\begin{proof}
  For the sake of contradiction, assume $t_2 \in I_2$, but $t_1 \notin
  M_{12}$. As $t_2 < t_1$, we have that $t_1 > \rep(M_{12}) =
  \rep(I_1)$. But then $S$ is not a solution as $S\cap I_1 = \emptyset$
  by Claim~\ref{cl:no-intermediate-value} and the fact that
  $t_2<\lep(I_1)$.
\end{proof}

\begin{boldclaim}\label{cl:inter-t2}
If $I'$ is an interval with $t_2\in I'$, then $I'\in F_2\cup F_1$.
\end{boldclaim}
\begin{proof}
  First, suppose $I'$ is a leader. As every leader has at least two
  followers when $I$ is scanned, $I'$ has two followers whose left
  endpoint is larger than $\rep(I')\ge t_2$ (by Property
  \ref{obs:follower-not-intersect-leader}) and smaller than $\lep(I)\le
  t_1$ (by \RedIncl). Thus, at least one of them is included in the
  interval $(t_2,t_1)$ by Property \ref{obs:followers-no-intersect},
  which contradicts $S$ being a solution by Claim
  \ref{cl:no-intermediate-value}.

Similarly, if $I'$ is a follower of a popular leader, but not among the last two followers of any popular leader, Claim \ref{cl:no-intermediate-value} leads to a contradiction as well.

Finally, if $I'$ is a follower, but has no popular leader, then it is to
the left of some popular leader, and thus
to the left of $t_2$.
\end{proof}

\noindent
Consider the set $T_2$ of intervals that intersect $t_2$. By Claim \ref{cl:inter-t2}, we have that $T_2 \subseteq F_2\cup F_1$. For every interval $I'\in T_2\cap F_2$, the corresponding merged interval of $\mathcal{I}'$ intersects $t_1$ by Claim \ref{cl:t1intersectsmerged}. For every interval $I'\in T_2\cap F_1$, and every interval $I''\in F_2$ with which $I'$ is merged, $S$ contains some value $x\in I''$ with $x<t_2$. Thus, $S':=S\setminus \{t_2\}$ is a solution for $\mathcal{I}'$.
\end{proof}

\begin{lemma}
If $S'$ is a nice solution for $\mathcal{I}'$, then there exists a solution $S$ for $\mathcal{I}$ with $S'\subseteq S$.
\end{lemma}
\begin{proof}
As in the previous proof, consider the step where the kernelization algorithm applies the \textbf{merge} operation. Recall that the currently scanned interval is $I$. Let $F_2$ and $F_1$ denote the set of all intervals that are the second-last and last follower of a popular leader before merging, respectively. Let $M$ denote the set of merged intervals. %Again, every interval of $F_1\cup F_2\cup M$ is required.

By Lemma \ref{lem:lastfollower-commonvalue}, every interval of $M$ intersects $\lep(I)$. On the other hand, every interval of $\mathcal{I}'$ whose right endpoint intersects $I$ is in $M$, by construction. Thus, $S'$ contains the right endpoint of some interval of $M$. Let $t_1$ denote the smallest such value, and let $I_1$ denote the interval of $\mathcal{I}$ with $\rep(I_1)=t_1$ (due to \RedIncl, there is a unique such interval). Let $I_2$ denote the interval of $\mathcal{I}$ with the smallest right endpoint such that there is a leader $L$ whose second-last follower is $I_2$ and whose last follower is $I_1$, and let $t_2:=\rep(I_2)$.

\begin{boldclaim}
Let $I_1'\in F_1$ and $I_2'\in F_2$ be two intervals from $\mathcal{I}$ that are merged into one interval $M_{12}'$ of $\mathcal{I}'$. If $t_1\in M_{12}'$, then $t_2\in I_2'$.
\end{boldclaim}
\begin{proof}
  For the sake of contradiction, suppose $t_1 \in M_{12}'$ but $t_2 \notin I_2'$.  We consider two
  cases. In the first case, $I_2' \subseteq (t_2,\lep(I_1'))$. But then,
  $I_2'$ would have become a follower of $L$, which contradicts that
  $I_1$ is the last follower of $L$. In the second case,
  $\rep(I_2')<t_2$. But then, $I_1$ is a follower of the same leader as
  $I_1'$, as $\lep(I_1)\le \lep(I_1')$, and thus $I_1=I_1'$. By the
  definition of $I_2$, however, $t_2=\rep(I_2)\le \rep(I_2')$, a
  contradiction.
\end{proof}

\noindent
By the previous claim, a solution $S$ for $\mathcal{I}$ is obtained from a solution $S'$ for $\mathcal{I}'$ by setting $S:=S'\cup \{t_2\}$.
\end{proof}

\noindent
After having scanned all the intervals, Reduction Rules \RedIncl, \RedDom, and \RedUnit are applied again, and we have already proved their correctness.

\smallskip
Thus, the kernelization algorithm returns an equivalent instance. To bound the kernel size by a polynomial in $k$, let $\mathcal{I}^*=(V^*,D^*,\dom^*,N^*)$ be the instance resulting from applying the kernelization algorithm to an instance $\mathcal{I}=(V,D,\dom,N)$.

\begin{property}\label{obs:nboptint}
 The instances $\mathcal{I}$ and $\mathcal{I}^*$ have at most $2k$ optional intervals.
\end{property}

\noindent
Property \ref{obs:nboptint} holds for $\mathcal{I}$ as every optional interval of a variable $x$ is adjacent to at least one hole and each hole is adjacent to two optional intervals of $x$.
It holds for $\mathcal{I}^*$ as the kernelization algorithm introduces no holes.

\begin{lemma}\label{lem:nbleaders}
 The instance $\mathcal{I}^*$ has at most $4k$ leaders.
\end{lemma}
\begin{proof}
  Consider the unique step of the algorithm that creates leaders. An
  optional interval is scanned, the algorithm continues scanning
  intervals until scanning a required interval, and all these scanned
  intervals become leaders. As every interval is scanned only once, we have that for
  every optional interval there are at most $2$ leaders. By Property
  \ref{obs:nboptint}, the number of leaders is thus at most $4k$.
\end{proof}

\begin{lemma}\label{lem:nbfollowers}
 Every leader has at most $4k$ followers.
\end{lemma}
\begin{proof}
Consider all steps where a newly scanned interval becomes a follower, but is not merged with another interval. In each of these steps, the popular leader $L_r$ with the rightmost right endpoint either
\begin{itemize}%[topsep=0pt, partopsep=0pt, itemsep=0pt]
\item[(a)] has no follower and intersects $I$, or
\item[(b)] has no follower and does not intersect $I$, or
\item[(c)] has one follower and intersects $I$.
\end{itemize}
Now, let $L$ be some leader and let us consider a period where no optional interval is scanned.
Let us bound the number of intervals that become followers of $L$ during this period without being merged with another interval.
If the number of followers of $L$ increases in Situation (a), it does not increase in Situation (a) again during this period, as no other follower of $L$ may intersect $I$.
After Situation (b) occurs, Situation (b) does not occur again during this period, as $I$ becomes a follower of $L_r$.
Moreover, the number of followers of $L$ does not increase during this period in Situation (c) after Situation (b) has occurred, as no other follower of $L$ may intersect $I$.
After Situation (c) occurs, the number of followers of $L$ does not increase in Situation (c) again during this period, as no other follower of $L$ may intersect $I$.
Thus, at most $2$ followers are added to $L$ in each period. As the first scanned interval is optional, Property \ref{obs:nboptint} bounds the number of periods by 
$2k$. Thus, $L$ has at most $4k$ followers.
\end{proof}

\noindent
As, by Claim \ref{cl:atleastoneleader}, every interval of
$\mathcal{I}^*$ is either a leader or a follower of at least one
leader, Lemmas \ref{lem:nbleaders} and
\ref{lem:nbfollowers} imply that
$\mathcal{I}^*$ has $O(k^2)$ intervals, and thus
  $|X^*|=O(k^2)$. Because of Reduction
  Rule \RedDom, every value in $D^*$ is the right
endpoint and the left endpoint of some interval, and thus,
  $|D^*|=O(k^2)$.

\smallskip

This bounds the kernel size, and we will now show that the algorithm can be implemented to run in linear time.
First, using a counting sort algorithm with satellite data (see, e.g.,
\cite{CormenLeisersonRivestStein09}), the initial sorting of the $n+k$
intervals can be done in time $O(n+|D|+k)$. To facilitate the
application of \RedIncl, counting sort is used a second time to also
sort by increasing left endpoint the sets of intervals with coinciding
right endpoint.
An optimized implementation applies \RedIncl, \RedDom
and \RedUnit simultaneously in one pass through the intervals, as one
rule might trigger another. To guarantee a linear running time for the
scan-and-merge phase of the algorithm, only the first follower of a
leader stores a pointer to the leader; all other followers store a
pointer to the previous follower.
This proves Theorem \ref{thm:kernel}.
%We omit the formal details about the
%implementation and running time analysis of the kernelization algorithm
%and arrive at our main theorem.

\section*{Acknowledgments}
Both authors acknowledge support by the European Research Council,
grant reference~239962 (COMPLEX REASON).
Serge Gaspers is the recipient of an Australian Research Council Discovery Early Career Researcher Award (project number DE120101761).
NICTA is funded by the Australian
Government as represented by the Department of Broadband,
Communications and the Digital Economy and the Australian Research
Council through the ICT Centre of Excellence program.

\bibliographystyle{plain}
\bibliography{literature}    
\end{document}